\documentclass[twoside,11pt]{article}

\usepackage{jmlr2e}

\usepackage[english]{babel} 
\usepackage{bookmark}
\usepackage{graphicx} 
\usepackage{hyperref} 
\usepackage{natbib}
\usepackage{amsmath,amssymb,bbm,amsfonts} 
\usepackage{url}
\usepackage[toc,page]{appendix}
\usepackage{color}
\usepackage{comment}
\usepackage{booktabs}

\newcommand{\R}{\mathbb{R}}
\newcommand{\N}{\mathbb{N}}

\newcommand{\X}{\mathcal{X}}
\newcommand{\1}{\mathbbm{1}}
\renewcommand{\P}{\mathbb{P}}
\newcommand{\E}{\mathbb{E}}

\newcommand{\Q}{\mathbb{Q}}

\newcommand{\qed}{$\hfill\blacksquare$}

\DeclareMathOperator*{\argmin}{arg\,min}

\newcommand{\veit}[1]{{\color{black}{\textbf{Veit:} #1}}}

\ShortHeadings{Connections between the Nystr\"om and Sparse Variational Gaussian Processes}{Wild, Kanagawa and Sejdinovic}
\firstpageno{1}

\begin{document}

\title{Connections and Equivalences between the Nystr\"om Method and Sparse Variational Gaussian Processes}

\author{\name Veit Wild \email veit.wild@keble.ox.ac.uk \\
       \addr Department of Statistics, University of Oxford, UK 
       \AND
       \name Motonobu Kanagawa \email motonobu.kanagawa@eurecom.fr \\
       \addr Data Science Department, EURECOM,  France 
       \AND
       \name Dino Sejdinovic \email 
       dino.sejdinovic@adelaide.edu.au \\
       \addr School of Computer and Mathematical Sciences, University of Adelaide, Australia
       }
       
\editor{}

\maketitle

\begin{abstract}
We investigate the connections between sparse approximation methods for making kernel methods and Gaussian processes (GPs) scalable to \textcolor{black}{large-scale} data, focusing on the Nystr\"om method and the Sparse Variational Gaussian Processes (SVGP). 
While sparse approximation methods for GPs and kernel methods share some algebraic similarities, the literature lacks a deep understanding of how and why they are related. This may pose an obstacle to the communications between the GP and kernel communities, making it difficult to transfer results from one side to the other. Our motivation is to remove this obstacle, by clarifying the connections between the sparse approximations for GPs and kernel methods. In this work, we study the two popular approaches, the Nystr\"om and SVGP approximations, in the context of a regression problem, and establish various connections and equivalences between them. In particular, we provide an RKHS interpretation of the SVGP approximation, and show that the Evidence Lower Bound of the SVGP contains the objective function of the Nystr\"om approximation, revealing the origin of the algebraic equivalence between the two approaches. We also study recently established convergence results for the SVGP and how they are related to the approximation quality of the Nystr\"om method.
\end{abstract}


\tableofcontents

\section{Introduction}

Gaussian processes (GPs) and kernel methods are the two principled learning approaches that make use of {\em positive definite kernels}, and have been studied extensively in statistics and machine learning. On one hand, GP-based approaches \citep{RasmussenWilliams} employ a kernel to induce the corresponding GP, in order to define a prior distribution of the ground-truth latent function of interest. Given data, Bayes' rule is then applied to obtain the posterior distribution of the latent function. 
On the other hand, kernel methods \citep{scholkopf2002learning} make use of a kernel to induce the corresponding Reproducing Kernel Hilbert Space (RKHS) as a ``hypothesis space.'' Given data, empirical risk minimization is then performed in the RKHS to estimate the ground-truth function of interest. Although the GP and kernel approaches have different modeling philosophies, there are indeed deep connections and equivalences between them, which extend beyond a superficial similarity \citep{Par61,kimeldorf1970correspondence,Berlinet2004,kanagawa2018gaussian}.

The elegance of the GP and kernel approaches is that the infinite-dimensional learning problems can be reduced to the corresponding finite-dimensional problems. 
However, this comes with a cost: a naive implementation for supervised learning usually leads to the computational complexity that is cubic in the data size.
This unfavorable scaling property has motivated the developments of several approximation methods to make the GP and kernel approaches scalable. {\em Sparse approximation} methods, which approximate the solution of interest using a set of input points smaller than training data, are among the most popular and successful approximation approaches.
These approaches have been studied since the earliest developments of the GP and kernel approaches \citep[e.g.,][]{williams2001using,csato2002sparse,smola2000sparse, seeger2003fast}.  

As the GP and kernel communities grow, sparse approximation methods for either approach tend to be developed independently to those for the other approach.
For instance, consider the Sparse Variational Gaussian Process (SVGP) approach of \citet{titsias2009variational,titsias2009techreport}, which is one of the most successful and widely used sparse approximation methods for GPs.
The SVGP is derived in the framework of variational Bayesian inference, so that the sparse approximation is to be chosen to minimize the KL divergence to the exact GP posterior. As such, the developments in SVGP \citep[e.g.,][]{hensman2013gaussian,hensman2015scalable,matthews2016sparse,burt2019rates,Rossi21sparse} have proceeded almost independently of the corresponding literature on sparse approximations for kernel methods.  Similarly, the recent advances in using and understanding the Nystr\"om method \citep{williams2001using}, which is one of the most popular sparse approximations in kernel methods, have been made independently to those of sparse GP approximations. The majority of these advances focus on an efficient approximation of the kernel matrix \citep[e.g.,][]{drineas2005nystrom,belabbas2009spectral,gittens2016revisiting,Derezinski2020Improved} or empirical risk minimization in the RKHS with a reduced basis \citep[e.g,][]{bach2013sharp,alaoui2015fast,rudi2015less,rudi2017falkon,Meanti_NeurIPS2020_kernel}. 
This separation of two lines of research is arguably due to the difference in the notations and modelling philosophies of GPs and kernel methods. 
The separation makes it difficult to transfer useful and interesting results from one side to the other, and the communities might have missed an important advance that may have been obtained otherwise. The motivation of the current work is to overcome this potential difficulty by bridging the two lines of research.

{\color{black}
This work investigates connections between sparse approximation methods for GPs and kernel methods, focusing on two fundamental methods: the SVGP and Nystr\"om approximations.  We consider the regression setting, where the exact solutions of Gaussian process regression (GPR) and kernel ridge regression (KRR) are to be approximated.  
Both the SVGP and Nystr\"om approximations use $m$ input locations for approximating the exact solutions, where $m$ is one's choice and smaller than the training data size $n$.  
They can be computed in the $O(m^2 n + m^3)$ complexity and thus can be much faster than the naive implementations requiring the $O(n^3)$ complexity.

The SVGP defines a class of GPs (the variational family used for approximation), parameterized by $m$ input locations called {\em inducing points} and by other parameters defining the means and covariances of the function values at the inducing points.  
For fixed inducing inputs, the mean and covariance parameters are obtained by minimizing the Kullback-Leibler (KL) divergence between the thus parametrized GP and the exact posterior of GPR, or equivalently by maximizing the Evidence Lower Bound (ELBO); the resulting optimal parametrized GP is the SVGP approximation.
On the other hand, the Nystr\"om method uses $m$ inducing points (or  {\em landmark points}) to define a $m$-dimensional subspace of the RKHS, and solves the regularized least-squares problem of KRR in this subspace; the resulting solution is the Nystr\"om approximation.

We first show that the ELBO objective for the parametric mean function of the SVGP is equivalent to the objective function for the Nystr\"om KRR (Section \ref{sec:eq_opt_mean_cov}). Consequently, the optimized mean function of the SVGP approximation is identical to the Nystr\"om KRR estimator.  This equivalence shows that the SVGP mean function solves a regularized least squares problem on the subspace in the RKHS spanned by inducing points.
 It is parallel to the well-known equivalence between GPR and KRR \citep{kimeldorf1970correspondence,kanagawa2018gaussian}, which the SVGP and Nystr\"om respectively approximate.

RKHS interpretations are also provided for the covariance function of the SVGP (Section \ref{sec:var_mean_cov}).  The ELBO objective for the parametric covariance function of the SVGP is analyzed  in terms of the RKHS geometry. The analysis elucidates how the covariance function of the SVGP approximates the posterior covariance function of the exact GPR. Moreover, the variance function of the SVGP  is shown to be equal to the sum of the worst-case error of the Nystr\"om KRR in the RKHS and that of a kernel interpolant on inducing points.  This equivalence enables understanding how SVGP quantifies uncertainties from a function-space viewpoint and how the choice of the kernel impacts uncertainty estimates.

Lastly, connections are established for the approximation quality of the SVGP and Nystr\"om approximations.   It is shown that the approximation error of the SVGP (as measured by the KL divergence) is closely related to the approximation error of the Nystr\"om (as measured by the regularized empirical risk). 
By leveraging this connection, we demonstrate how upper and lower bounds for the  SVGP approximation error by \citet{BurtJMLR20} can be translated into the corresponding bounds for the Nystr\"om approximation error (Sections \ref{sec:theory-connection-Nystrom} and \ref{sec:lower-bound-approx}). 
Similarly, we show how an approximation error bound for the Nystr\"om as measured by the RKHS distance, which is novel and may be of independent interest, can be translated into  an approximation error bound for the {\em derivatives} of the mean function of the SVGP (Section \ref{sec:RKHS-error-bound}).

To summarize, our contributions are in providing novel interpretations for the SVGP and Nystr\"om approximations by establishing their connections and equivalences.  These interpretations will help understand either approach from the perspective of the other approach, enable translating results from one side to the other, and encourage new algorithmic developments.
}

\textcolor{black}{
This paper is organized as follows. We provide relevant background on Gaussian processes and kernel methods in Section \ref{sec:background}, and on the SVGP and Nystr\"om approximations in Section \ref{sec:sparse-approx-methods}. We investigate the equivalences between the SVGP and Nystr\"om approximations in Section \ref{chapter_large_scale_approx}, and study the connections between their approximation properties in Section \ref{chapter_convergence_bounds}. We conclude in Section \ref{sec:conclusion}. The appendix provides proofs for the main theoretical results, and a derivation of the optimal variational parameters of the SVGP based on \citet{khan2021bayesian} on the Bayesian learning rule.
}



\subsection{Notation} \label{sec:notation}
We use the following notation in this paper.
Let $\mathbb{N}$ be the set of natural numbers, $\mathbb{R}$ be the real line, and $\mathbb{R}^d$ for $d \in \mathbb{N}$ be the $d$-dimensional Euclidean space. For any $v \in \mathbb{R}^d$, $\| v \|$ denotes the Euclidean norm. 

Let $\mathcal{X}$ be a nonempty set.
For a function $f : \mathcal{X} \to \mathbb{R}$ and $X := (x_1,\dots,x_n) \in \mathcal{X}^n$ with $n \in \mathbb{N}$, denote by $f_X$ the $n$-vector consisting of function values evaluated at points in $X$: $f_X := (f(x_1), \dots, f(x_n))^\top \in \mathbb{R}^n$. 
Similarly, for a function with two arguments $k: \mathcal{X} \times \mathcal{X} \to \mathbb{R}$,  and  $X:=(x_1,...,x_n) \in \X^n$ and $Z:=(z_1,...,z_m)\in \X^m$ with $n, m \in \mathbb{N}$, define $k_{XZ}\in \R^{n \times m}$ by ($k_{XZ})_{i,j} = k(x_i, z_j)$ for $i=1,\dots,n$, $j = 1,\dots, m$.
For a matrix $X:=(x_1,...,x_n) \in \X^n$, let $k_X(x) := (k(x_1,x), \dots, k(x_n, x))^\top \in \mathbb{R}^n$ for any $x \in \mathcal{X}$ and denote by $k_X(\cdot)$ the vector-valued function $x \in \mathcal{X} \mapsto k_X(x) \in \mathbb{R}^n$.

For a symmetric matrix $\Sigma$, denote by $\Sigma \succ 0$ and $\Sigma \succeq 0$ that $\Sigma$ is positive definite and positive semi-definite, respectively.
  For $\mu \in \mathbb{R}^n$ and $\Sigma \in \mathbb{R}^{n \times n}$ with $\Sigma \succeq 0$, denote by $\mathcal{N}(\mu, \Sigma)$ the Gaussian distribution on $\mathbb{R}^n$ with mean vector $\mu$ and covariance matrix $\Sigma$.
 Let $\mathcal{N}( \cdot \mid  \mu, \Sigma) $ be its probability density function.
For a matrix $A \in \mathbb{R}^{n \times n}$, ${\rm tr}(A)$ and $\det (A)$ denote its trace and determinant, respectively.



Let $\mathcal{Y}$ be a measurable space, $Y \in \mathcal{Y}$ be a random variable, and $\P$ be a probability measure on $\mathcal{Y}$. 
We write $Y \sim \P$ to mean that $Y$ follows $\P$.
For a measurable function $g:\,\mathcal{Y} \to \mathbb{R}$, denote by $\int g(y)  d\P(y)$ its integral with respect to $\P$ and by $\mathbb{E}[g(Y)]$ the expectation of $g(Y)$.
When $\P$ has a density function $p: \mathcal{Y} \to \mathbb{R}$ with respect to a reference measure $\lambda$ on $\mathcal{Y}$ (e.g., the Lebesgue measure when $\mathcal{Y} = \mathbb{R}^n$), the integral is denoted by $\int g(y) p(y)  d\lambda(y)$.

\section{Kernels and Gaussian Processes for Regression} \label{sec:background}

This section briefly reviews reproducing kernel Hilbert spaces (RKHS) and Gaussian processes (GP).
In particular, we focus on the respective approaches to  {\em regression}, namely kernel ridge regression (KRR) and Gaussian process regression (GPR), which are reviewed in Sections \ref{sec:Background-KRR} and \ref{ch_GPR}, respectively. 
We review their connections in Section \ref{sec:connections_exact}.

We first describe the regression problem.
Let $\mathcal{X}$ be a non-empty set.
Suppose we are given $n \in \mathbbm{N}$ paired observations 
\begin{align*}
(x_1,y_1),\ldots,\allowbreak(x_n,y_n) \in \mathcal{X} \times \mathbbm{R}.    
\end{align*}
We assume that there exists a function $f_0: \mathcal{X} \to \mathbb{R}$ such that
\begin{equation} \label{eq:regression}
y_i = f_0(x_i) + \varepsilon_i, \quad i = 1,\dots,n.
\end{equation}
where $\varepsilon_1, \dots, \varepsilon_n \in \mathbb{R}$ are independent, zero-mean, noise variables.
This $f_0$ is called {\em regression function}.
The task of regression is to estimate (or learn) $f_0$ from the training data $(x_i,y_i)_{i=1}^n$.
We will often write $X := (x_1,\dots,x_n) \in \mathcal{X}^n$ and $y := (y_1, \dots, y_n) \in \mathbb{R}^n$.


\subsection{Kernel Ridge Regression (KRR)} 
\label{sec:Background-KRR}

\subsubsection{Kernels and RKHSs}
We review here basics of kernels and RKHSs. 
For details, we refer to \citet{scholkopf2002learning,hofmann2008kernel,SteChr2008}.

Let $\mathcal{X}$ be an arbitrary non-empty set. 
A symmetric function $k: \mathcal{X}\times \mathcal{X} \to \mathbb{R}$ is called  {\em positive definite kernel}, if for every $n \in \mathbb{N}$ and every $X= (x_1,...,x_n) \in \mathcal{X}^n$, the induced kernel matrix $k_{XX} = (k(x_i, x_j))_{i,j=1}^n \in \mathbb{R}^{n \times n}$  is positive semi-definite.
We may simply call such $k$ {\em kernel}.
Examples of kernels on $\mathcal{X}\subset \mathbb{R}^d$ include  the Gaussian or square-exponential kernel 
$k(x,x') = \exp(- \frac{\|x-x'\|^2}{\gamma^2} )$  for $ \gamma > 0 $, and polynomial kernels $k(x, x') = ( x^\top x' + c )^m$ for $m \in \mathbb{N} $ and $c \geq 0$. 


By the Moore-Aronszajn theorem \citep{aronszajn1950theory},  
  any positive definite kernel $k$ is uniquely associated with a Hilbert space $\mathcal{H}_k$
of real-valued functions  $f:\mathcal{X} \to \mathbb{R}$ called {\em reproducing kernel Hilbert space (RKHS)} that satisfies
	\begin{enumerate}
		\item $k(\cdot,x) \in \mathcal{H}_k$ for every $x \in \mathcal{X}$ and
		\item $f(x) = \langle f,k(\cdot,x)\rangle_{\mathcal{H}_k} \ $  for every $f \in \mathcal{H}_k$ and $x \in \mathcal{X}$\ ({\em reproducing property}),
	\end{enumerate}
	where $\left<f, g \right>_{\mathcal{H}_k}$ for $f, g \in \mathcal{H}_k $ denotes the inner product in $\mathcal{H}_k$, and  $k(\cdot, x)$ denotes the function of the first argument with $x$ being fixed: $x'\mapsto k(x', x)$, for $x'\in \mathcal{X}$.
   This $k(\cdot,x)$ is interpreted as a feature vector of $x$, and the kernel  $k(x, x') = \left<k(\cdot,x), k(\cdot,x')  \right>_{\mathcal{H}_k}$ as a similarity of inputs $x$ and $x'$.  
	The kernel $k$ is called {\em reproducing kernel} of $\mathcal{H}_k$.  
    
  The RKHS $\mathcal{H}_k$ consists of functions that can be approximated arbitrarily well by a linear combination of feature vectors of the form $\sum_{i=1}^n c_i k(\cdot, x_i)$ for some $c_1, \dots, c_n \in \mathbb{R}$ and $x_1, \dots, x_n \in \mathcal{X}$. In other words, the set of functions 	
   \begin{align*}
		&\mathcal{H}_0:=
		\Big\{ \sum_{i=1}^{n} \alpha_i k(\cdot, x_i) \mid \ n \in \mathbb{N}, \  \alpha_1,...,\alpha_n  \in
		\mathbb{R},\   x_1,...,x_n \in \mathcal{X}
	 \Big\}. 
	\end{align*}
    is a dense subset of $\mathcal{H}_k$ \citep[e.g.,][Section 2.2.1]{hofmann2008kernel}. 


\subsubsection{Regression Approach}\label{sec:KRR}

Kernel ridge regression (KRR) is an approach to  regression using a kernel $k$ and its RKHS $\mathcal{H}_{k}$. 
The KRR estimator $\hat{f}$ of the regression function $f_0$ in \eqref{eq:regression} is defined as the solution of the following {\em regularized least-squares} problem 
	\begin{equation}\label{eq_krr}
		\hat{f}= \underset{f \in \mathcal{H}_k}{\text{argmin }} \frac{1}{n} \sum_{i=1}^{n}  (y_i-f(x_i))^2 + \lambda ||f||^2_{\mathcal{H}_k},
	\end{equation}	
	where $\lambda > 0$ is a regularization constant.
Since the RKHS norm $\| f \|_{\mathcal{H}_k}$ quantifies the smoothness of $f$ and becomes large for non-smooth $f$,  the regularization term in (\ref{eq_krr}) makes the optimal solution $\hat{f}$ smooth; the constant $\lambda$ determines the strength of this smoothing effect. 


	Let $X:=(x_1,...,x_n) \in \mathcal{X}^n$ and $y:=(y_1,...,y_n)^\top \in \mathbb{R}^n$. 
    By the representer theorem \citep{scholkopf2001generalized}, the solution $\hat{f}$ is given as a linear combination of $k(\cdot,x_1)$,  $\dots,k(\cdot,x_n)$. 
    Hence, the optimization problem \eqref{eq_krr} reduces to that of the coefficients of the linear combination.
As a result, the estimator is given by
\begin{align}\label{eq_form_2}
	\hat{f}=\sum_{i=1}^{n} \alpha_i  k(\cdot,x_i), \quad \text{with}\quad \alpha := (\alpha_1,...,\alpha_n)^\top := 	(k_{XX}+n\lambda I_n)^{-1} y \in \mathbb{R}^n,
	\end{align}
where $I_n \in \mathbb{R}^{n \times n}$ is the identity matrix.
The prediction of KRR at any $x \in \mathcal{X}$ is compactly written as 
\begin{equation} \label{eq:KRR-compact}
    	\hat{f}(x)  = k_X(x)^\top \alpha = k_X(x)^\top	(k_{XX}+n\lambda I_n)^{-1} y,
\end{equation}
	where $k_X(x) = (k(x_1,x), \dots, k(x_n, x))^\top \in \mathbb{R}^n$.	

\subsection{Gaussian Process Regression (GPR)}\label{ch_GPR}

\subsubsection{Gaussian Processes}
Gaussian processes (GPs) are one of the main workhorses of Bayesian nonparametric statistics and machine learning, as they can be used to place a prior distribution over functions.
See \citet{RasmussenWilliams} for more details.

Let $\mathcal{X}$ be a non-empty set, $m : \mathcal{X} 
\to \mathbb{R}$ a function and $k: \mathcal{X} \times \mathcal{X} \to \mathbb{R}$ a positive definite kernel. 
A random function $F: \mathcal{X} \to \mathbb{R}$ is called {\em Gaussian process (GP)} with mean function $m$ and covariance kernel $k$, if for all $n \in \N$ and all $X=(x_1,..., x_n) \in \X^n$, the random vector $F_X:= (F(x_1), \dots, F(x_n) )^\top \in \R^n$ satisfies $F_X \sim \mathcal{N}(m_X, k_{XX})$,
i.e., $F_X$ follows the Gaussian distribution with mean vector $m_X = (m(x_1), \dots, m(x_n))^\top \in \mathbb{R}^n$ and covariance matrix $k_{XX} = (k(x_i,x_j))_{i,j=1}^n \in \mathbb{R}^{n\times n}$.
In this case, we write $F \sim GP(m, k)$.




For any function $m: \X \to \mathbb{R}$ and kernel $k: \X \times \X \to \mathbb{R}$, there exists a GP whose mean function is $m$ and covariance function is $k$. 
Therefore, by choosing $m$ and $k$, one can implicitly define the corresponding GP,  $F \sim GP(m,k)$.
This is how a GP is used to define a prior distribution in Bayesian nonparametrics.



\subsubsection{Regression Approach}\label{sec:regression_approach}
{\em Gaussian process regression (GPR)} is a Bayesian nonparametric approach to the regression problem. 
In GPR, the regression function $f_0$ in Eq.~\eqref{eq:regression} is the quantity of interest and modeled as a random function $F$. 
The prior distribution is given by a GP 
\begin{align} \label{eq:GP-prior}
    F \sim GP(m,k),
\end{align}
where the mean function $m$ and covariance function $k$ are chosen to encode \textcolor{black}{one's prior knowledge/assumption about the regression function $f_0$.}
The observation model of $F$ for the observations $y = (y_1, \dots, y_n)^\top$ is given by
\begin{align} \label{eq:likelihood}
y_i = F(x_i) + \varepsilon_i, \quad i = 1,\dots, n,
\end{align} 
where  $\varepsilon_i \sim \mathcal{N} (0, \sigma^2)$ is an independent Gaussian noise with variance $\sigma^2 > 0$.

By Bayes' rule, the posterior distribution of $F$ given $y$, under the prior \eqref{eq:GP-prior}, is given by another GP  \citep[e.g.,][]{RasmussenWilliams}:
\begin{align*}
    F \mid y \sim GP(\bar{m},\bar{k}),
\end{align*}
where $\bar{k}: \mathcal{X}\times\mathcal{X} \to \mathbb{R}$ and $\bar{m}: \mathcal{X} \to \mathbb{R}$ are defined as	
\begin{align}
		\bar{m}(x)&:= m(x) + k_{X}(x)^\top(k_{XX}+\sigma^2 I_n )^{-1} (y-m_X), \label{eq_bar_m} \\
		\bar{k}(x,x')&:= k(x,x') - k_{X}(x)^\top(k_{XX}+\sigma^2 I_n )^{-1} k_{X}(x') \label{eq_bar_k},
\end{align}
where $k_{X}(x) = (k(x_1,x), \dots, k(x_n, x))^\top \in \mathbb{R}^n$.
We call $GP(\bar{m},\bar{k})$ the {\em posterior GP}, $\bar{m}$ the {\em posterior mean function} and $\bar{k}$ the {\em posterior covariance function}.
We use the following notation for the probability measure of the posterior GP:  
	\begin{equation} \label{eq:GP-posterior}
		 \P^{F|y}  := GP(\bar{m},\bar{k}).
	\end{equation}
The posterior mean  $\bar{m}(x)$ serves as a predictor of the ground-truth function value $f_0(x)$, while the posterior variance $ \bar{k}(x,x)$ quantifies its uncertainty.

\subsection{Connections between the Exact Solutions for KRR and GPR}\label{sec:connections_exact}

\textcolor{black}{ 
It is well-known that there are connections between KRR and GPR \citep{kimeldorf1970correspondence,kanagawa2018gaussian}, which are briefly summarized here.}

\textcolor{black}{
First, there is an equivalence between the predictors of KRR and GPR.
Suppose that the same kernel $k$ is used in KRR and GPR, and that
 the prior mean function of GPR is zero, $m(x) = 0$. Then the posterior mean function $\bar{m}(x)$ of GPR in \eqref{eq_bar_m} is identical to the KRR estimator  $\hat{f}(x)$ in \eqref{eq:KRR-compact}, provided that the regularization constant $\lambda$ in KRR satisfies $\sigma^2 = n \lambda$.  
 This equivalence provides a Bayesian interpretation of the KRR estimator, and a least-squares interpretation of the GPR posterior mean function. }

\textcolor{black}{
Second,  the posterior variance $\bar{k}(x,x)$ of GPR can be interpreted as a {\em worst case error} of KRR predictions in the RKHS of a certain augmented kernel defined from the GP covariance kernel $k$ and the noise variance $\sigma^2$ \citep[Section 3.4]{kanagawa2018gaussian}.
This interpretation enables discussing how the RKHS is related to the uncertainty estimates of GPR. 
Further connections between the two approaches can be found in \citet{kanagawa2018gaussian}.}

\textcolor{black}{
This paper investigates whether these parallels between KRR and GPR extend to their {\em sparse approximations}, which are reviewed in the next section.}


{\color{black}

}

\section{Sparse Approximation Methods}
\label{sec:sparse-approx-methods}

\textcolor{black}{
This section reviews sparse approximation methods for KRR and GPR, specifically the {\em Nyst\"om method} in Section \ref{sec:Nystrom} and  {\em Sparse Variational Gaussian Processes (SVGP)} in Section \ref{sec:SVGP-GP}. 
While closed-form expressions are available for KRR and GPR, their computations require the $O(n^3)$ complexity for training data size $n$, which is challenging when $n$ is large. 
A sparse approximation method reduces this computational complexity by approximating the exact solution based on a smaller set of input points $z_1, \dots, z_m \in \X$, where $m  < n$. 
The Nystr\"om and SVGP approximations are commonly used sparse approximation methods for kernel methods and GPs, respectively. We review them here as a preliminary  for our investigation of their connections in the subsequent sections. 
}


\subsection{Nystr\"om Approximation}\label{sec:Nystrom}
The Nystr\"om method was first proposed by \citet{williams2001using} for scaling up kernel-based learning algorithms.
It has been successfully used in a variety of applications including manifold leaning \citep{talwalkar2008large,talwalkar2013large}, computer vision \citep{fowlkes2004spectral,belabbas2009spectral}, and approximate sampling \citep{affandi2013nystrom}, to name a few. 
Recent studies make use of the Nystr\"om method to enable KRR to handle millions to billions of data points \citep[e.g.,][]{rudi2017falkon,Meanti_NeurIPS2020_kernel}.

We describe here the use of the Nystr\"om approximation in KRR. 
In particular, we consider a popular version classically known as the {\em subset of regressors} \citep[Chapter 7]{wahba1990spline}, which has been widely used both in practice and theory \citep[e.g.,][]{smola2000sparse,rudi2015less,rudi2017falkon,Meanti_NeurIPS2020_kernel}.   
As before, let $(x_i, y_i)_{i=1}^n \subset \mathcal{X} \times \mathbb{R}$ be training data, and let $X:=(x_1,...,x_n) \in \mathcal{X}^n$ and $y:=(y_1,...,y_n)^\top \in \mathbb{R}^n$.

For $m \in \mathbb{N}$, let $z_1, \dots, z_m \in \X$ be a set of input points based on which we approximate the KRR solution. 
These points $z_1, \dots, z_m$ are usually a subset of training input points $x_1, \dots, x_n$ in the kernel literature, but we allow for $z_1, \dots, z_m$ to be generic points in $\X$ for a later comparison with the GP counterpart.
Write $Z = (z_1, \dots, z_m) \in \mathcal{X}^m$.
Suppose that the kernel matrix $k_{ZZ} = (k(z_i, z_j) )_{i,j =1}^m \in \mathbb{R}^{m \times m}$ is invertible.

Let $M \subset \mathcal{H}_k$ be the finite dimensional subspace spanned by $k(\cdot,z_1), \dots, k(\cdot,z_m)$:
\begin{align}
  M:&=  {\rm span} (k(\cdot, z_1), \dots, k(\cdot, z_m)) :=\left\{ \sum_{j=1}^m \alpha_j k(\cdot,z_j) \mid \alpha_1 ,...,\alpha_m \in \R  \right\} \label{eq_span_M}.
\end{align}
We replace the hypothesis space $\mathcal{H}_k$ in the KRR objective function \eqref{eq_krr} by this subspace $M$, and define its solution $\bar{f}$ as the Nystr\"om approximation of the KRR solution $\hat{f}$:
	\begin{equation} \label{eq:Nystrom-opt}
	\bar{f} :=	\underset{f \in M}{\argmin } \ \frac{1}{n} \sum_{i=1}^{n} \big(y_i - f(x_i) \big)^2 + \lambda \|f \|^2_{\mathcal{H}_k}.
	\end{equation}
In other words, we approximately solve the minimization problem of KRR by searching for the solution  of the form
\begin{equation*}
	f = \sum_{i=1}^{m} \beta_i k(\cdot, z_i) = k_Z(\cdot)^\top \beta
\end{equation*}
for some coefficients $\beta := (\beta_1,\dots, \beta_m)^\top \in \R^m$, where $k_Z(\cdot) := (k(\cdot, z_1), \dots, k(\cdot, z_m))^\top$. 
Inserting this expression in \eqref{eq:Nystrom-opt}, the optimization problem now becomes
\begin{equation*}
\min_{ \beta \in \R^m } \frac{1}{n} \| y-k_{XZ}\beta \|^2 + \lambda \beta^\top k_{ZZ} \beta,
\end{equation*}
where $k_{XZ} \in \mathbb{R}^{n \times m}$ with $(k_{XZ})_{i,j} = k(x_i, z_j)$ and $k_{ZZ} \in \mathbb{R}^{m \times m}$ with $(k_{ZZ})_{i,j} = k(z_i, z_j)$.
Taking the first order derivative with respect to $\beta$ leads to the condition 
\begin{equation*}
-\frac{2}{n} k_{ZX}y +  \frac{2}{n} k_{ZX} k_{XZ} \beta + 2 \lambda k_{ZZ} \beta 	= 0,
\end{equation*}
which is satisfied for 
\begin{equation*}
\beta = \big(k_{ZX} k_{XZ} + n \lambda k_{ZZ} \big)^{-1} k_{ZX} y .
\end{equation*}
This leads to the following expression of the Nystr\"om approximation:
\begin{equation} \label{eq_nystroem}
	\bar{f}(x) = k_{Z}(x)^\top (n\lambda k_{ZZ} + k_{ZX} k_{XZ})^{-1} k_{ZX} y.
\end{equation}
This approximation can be computed with the complexity of  $\mathcal{O}(nm^2+m^3)$  instead of $\mathcal{O}(n^3)$, since the inversion of a $n\times n$ matrix is replaced by that of a $m \times m$ matrix. \textcolor{black}{(The complexity $O(nm^2)$ is that of  the matrix multiplication $k_{ZX} k_{XZ}$, which is the dominating part as $m < n$.)} 
This grants significant computational gains, if $m$ is much smaller than $n$ and hence allows KRR to be applied to large data sets. 
Of course, how to choose $m$ and the input points $z_1, \dots, z_m$ depends not only on the computational budget but also on how accurately $\bar{f}$ approximates the KRR solution $\hat{f}$. We discuss this issue in Section \ref{chapter_convergence_bounds}.


\subsection{Sparse Variational Gaussian Processes}  \label{sec:SVGP-GP}

We review here the Sparse Variational Gaussian Process (SVGP) approach  by \cite{titsias2009variational} based on a measure-theoretic formulation\footnote{\textcolor{black}{The main reason for presenting in the measure-theoretic formulation is to treat the variational mean $m^\nu(x)$ and covariance  $k^\nu(x,x')$ introduced below as {\em functions}; this is needed for establishing the equivalences and connections to the RKHS counterparts, which are also functions.}} suggested by \cite{matthews2016sparse}.
There have been many works on sparse approximations for scaling up GP-based methods.
In a nutshell, there are two common approaches: either the generative model is approximated and inference is performed exactly \citep{seeger2003fast,snelson2006sparse, snelson2007local} or the generative model is left unaltered and inference is done approximately \citep{csato2002sparse,titsias2009variational}. 
In this work, we mainly focus on the SVGP approximation by \citet{titsias2009variational}, which is the latter approach. 
We refer to \citet{bauer2016understanding} and \citet[Part I]{bauer2020advances} for a systematic comparison of the two approaches.

Since we focus on the basic framework of \cite{titsias2009variational} and its comparison to the kernel counterpart in a regression setting, we do not discuss sparse variational GP approaches to the classification problem \citep{hensman2015scalable} and other (more recent) developments \cite[e.g.,][]{Hensman2015MCMC,hensman2017variational,dutordoir20-sparse,adam2020doubly,shi2019sparse,Rossi21sparse,tran21sparse}.
See e.g.,~\cite{leibfried2020tutorial} for an overview over variational GP approaches.

We first recall the setting of GPR using a measure-theoretic notation. 
As before, let $(x_i, y_i)_{i=1}^n \subset \mathcal{X} \times \mathbb{R}$ be training data and let $X = (x_1,\dots,x_n) \in \mathcal{X}^n$ and $y = (y_1,\dots,y_n)^\top \in \mathbb{R}^n$.
For simplicity, we assume the zero prior mean function, $m(x) = 0$.
We denote by $\P$  the probability measure of a Gaussian process $F \sim GP(0,k)$. 
For any finite set of points $D := (d_1, \dots, d_\ell) \in \mathcal{X}^\ell$ with $\ell \in \N$, let $\P_D$ be the corresponding distribution of $F_D := (F(d_1), \dots, F(d_\ell) )^\top$ on $\mathbb{R}^\ell$, which is $\P_D = \mathcal{N}(0,k_{DD})$ by definition.

\subsubsection{Variational Family}\label{subsec:var_family}
We first introduce a variational family of probability measures of functions on $\mathcal{X}$, from which we search for a computationally tractable approximation of the GP posterior $\P^{F|y} =  GP(\bar{m}, \bar{k})$ in \eqref{eq:GP-posterior}.
Let $m \in \mathbb{N}$ be fixed, and $\Gamma$ be a set of {\em variational parameters} defined by  
\begin{align*}
\Gamma := \{  \nu:= (Z, \mu, \Sigma) \mid ~ & Z := (z_1, \dots, z_m) \in \mathcal{X}^m, k_{ZZ}  \text{ is invertible, }\\ 
& \mu \in \mathbb{R}^m, \Sigma \in \mathbb{R}^{m \times m}_{\succ 0}  \}    
\end{align*}
where $\R^{m \times m}_{\succ 0}$ stands for symmetric and positive definite matrices in $\R^{m \times m}$. 
The points $Z = (z_1, \dots, z_m)$ are the so-called {\em inducing inputs}, based on which we approximate the posterior GP.
On the other hand, $\mu$ and $\Sigma$ are parameters for the distribution of function values at $z_1, \dots, z_m$.   

We then define a variational family 
    $\mathcal{Q}_\Gamma := \{ \Q^\nu \ | \ \nu \in \Gamma \}$
 as a set of Gaussian processes parametrized by the triplet\footnote{\textcolor{black}{Note that we assume that the kernel $k$ is fixed, and hence the variational family is parametrized by this triplet. If we consider the learning of the parameters of the kernel $k$ as well, then they may be included as parameters of the variational family.}} $\nu = (Z, \mu, \Sigma)$ defined as follows:
\begin{align}\label{eq_VGP_par_fam}
     \Q^\nu &:= GP(m^\nu, k^\nu), \\
     m^\nu(x) &:= k_{Z}(x)^\top k_{ZZ}^{-1} \mu, \label{eq:mean-variational} \\
     k^\nu(x,x') &:=  k(x,x') - k_{Z}(x)^\top k_{ZZ}^{-1} k_{Z}(x')    + k_{Z}(x)^\top k_{ZZ}^{-1} \Sigma k_{ZZ}^{-1} k_{Z}(x'). \label{eq:cov-variational}
\end{align}
Each variational distribution \eqref{eq_VGP_par_fam} is defined so as to satisfy the following properties, where $F^\nu \sim GP(m^\nu, k^\nu)$ denotes the corresponding GP sample function:
\begin{enumerate}
    \item The function values\footnote{Note that $F_Z^\nu$ is usually called {\em inducing variables} and is denoted with symbol ${\bf u}$ in the literature.} $F^\nu_Z := (F^\nu(z_1), \dots, F^\nu(z_m))^\top \in \mathbb{R}^m$ at the inducing inputs $z_1,\dots,z_m$ follow the Gaussian distribution with mean vector $\mu \in \mathbb{R}^m$ and covariance matrix $\Sigma \in \mathbb{R}^{m \times m}_{\succ 0}$, i.e., $F^\nu_Z \sim \mathcal{N}(\mu, \Sigma)$.
    We denote by $\Q^\nu_Z$ by the distribution of $F^\nu_Z$, i.e., $\Q^\nu_Z = \mathcal{N}(\mu, \Sigma)$.
    \item The conditional distribution of the process $F^\nu$  given $(z_i, F^\nu(z_i))_{i=1}^m$ is identical to the conditional distribution of $F \sim GP(m, k)$  given $(z_i, F(z_i))_{i=1}^m$:
    \begin{equation} \label{eq:Titsias-cond}
            F^\nu  \mid (z_i, F^\nu(z_i))_{i=1}^m \ \stackrel{d}{=} \ F   \mid (z_i, F(z_i))_{i=1}^m.
    \end{equation}
\end{enumerate}

\subsubsection{Evidence Lower Bound and Optimal Variational Parameters}\label{sec:ELBO_opt_para}

The aim of variational inference is to obtain a distribution $\Q^{\nu^*}$ from the variational family $\mathcal{Q}_\Gamma$ that best approximates the posterior measure $\P^{F|y}$ in terms of the Kullback-Leibler (KL) divergence, without explicitly computing the posterior. 
That is, we want to compute $\nu^* \in \Gamma$ such that
\begin{equation}
    \nu^* \in \arg\min_{\nu \in \Gamma} \  KL( \Q^\nu \ \| \ \P^{F|y}).
\end{equation}
where $KL( \Q^\nu \ \| \ \P^{F|y})$ is the KL divergence between $\Q^\nu$ and $\P^{F|y}$ defined by
$$
KL(\Q^\nu\ \| \ \P^{F|y}) := \int  \log \left( \frac{d\Q^\nu}{d\P^{F|y}} (f) \right) d\Q^\nu(f).  
$$
with $\frac{d\Q^\nu}{d\P^{F|y}}$ being the Radon-Nikodym derivative of $\Q_\nu$ with respect to $\P^{F|y}$, which exists by the construction of $\Q^\nu$ \citep[Section 3.3]{matthews2016sparse}.

\citet[Eq.~15]{matthews2016sparse} show that this KL divergence can be written as
\begin{align}
        &KL \big(\Q^{\nu} \ \| \ \P^{F|y} \big) 
        = \log p(y) - \mathcal{L}(\nu), \label{eq_VGPR_ELBO}
\end{align}
where $p(y)$ is the marginal likelihood, or the {\em Evidence}, of observing $y = (y_1,\dots,y_n)^\top$ under the prior $F \sim GP(m,k)$ and the observation model $y_i \sim \mathcal{N}(F(x_i), \sigma^2)$, while $\mathcal{L}(\nu)$ is the {\em Evidence Lower Bound (ELBO)} defined as
    \begin{align} \label{eq_def_L}
        \mathcal{L}(\nu) := - KL \big( \Q_Z^\nu \  \| \  \P_Z  \big) + \E_{F^\nu \sim \Q^\nu} \big[ \log p(y|F^\nu_X) \big],
    \end{align}
where
\begin{itemize}
    \item  $KL \big( \Q_Z^\nu \  \| \  \P_Z \big)$ is the standard KL divergence between $\Q_Z^\nu  = \mathcal{N} (\mu, \Sigma)$, which is the marginal distribution of $F^\nu_Z \in \mathbb{R}^m$ of the parameterized Gaussian process $F^\nu \sim \Q^\nu = GP(m^\nu, k^\nu)$ in \eqref{eq_VGP_par_fam}, and $\P_Z = \mathcal{N}(0, k_{ZZ})$, which is the marginal distribution of $F_Z \in \mathbb{R}^m$ of the prior Gaussian process $F \sim \P = GP(0,k)$:
\begin{align}
& KL \big( \Q_Z^\nu \  \| \  \P_Z  \big) = \int_{\mathbb{R}^m} \log \left( \frac{d\Q^\nu_Z}{d\P_Z} (f_Z) \right) d\Q^\nu_Z (f_Z) \nonumber \\    
& = \frac{1}{2} \left( {\rm tr} (k_{ZZ}^{-1} \Sigma) + \mu^\top k_{ZZ}^{-1} \mu - m + \log \left(  \frac{ {\rm det} (k_{ZZ})  }{ {\rm det} (\Sigma)  } \right) \right),  \label{eq:KL-Gaussians}
\end{align}
where the last identity is the well-known expression of the KL divergence between multivariate Gaussian densities (see, e.g., Appendix A.5 of \citealt{RasmussenWilliams}).

\item $\E_{F^\nu \sim \Q^\nu}[ \log p(y|F^\nu_X)]$ is the {\em marginal log likelihood} of observing $y = (y_1, \dots, y_n)^\top$ under the observation model $y_i = F^\nu(x_i) + \varepsilon_i$ with independent $\varepsilon_i \sim \mathcal{N}(0, \sigma^2)$ and the parametrized process $F^\nu \sim \Q^\nu$:
\begin{align}
    & \E_{F^\nu \sim \Q^\nu}[ \log p(y|F^\nu_X)]  = - n \log ( \sqrt{2\pi \sigma^2} ) - \E_{F^\nu \sim \Q^\nu}\Big[ \sum_{i=1}^n  \frac{ (y_i - F^\nu (x_i) )^2 }{2\sigma^2} \Big] \label{eq:ELBO-fit-term}
\end{align}
where $p(y|F^\nu_X) := \mathcal{N}(y ; F_X^\nu, \sigma^2 I_n )$ is the Gausian density of the likelihood function \eqref{eq:likelihood}.     
\end{itemize}

Since the marginal likelihood $p(y)$ under the original GP prior does not depend on the variational parameters $\nu$, the minimization of the KL divergence \eqref{eq_VGPR_ELBO} is equivalent to the maximization of the ELBO $\mathcal{L}(\nu)$ in \eqref{eq_def_L}. 
\citet[Eq.(10)]{titsias2009variational} show that,\footnote{See Appendix A of \citet{titsias2009techreport} or \citet[Section 2]{hensman2013gaussian} for the derivation. \textcolor{black}{We also show alternative derivations based on \citet{khan2021bayesian} in Appendix \ref{ap:alter_der}.} } for fixed inducing points $Z$, the optimal parameters $\mu^*$ and $\Sigma^*$ that maximize the ELBO are given analytically as
\begin{align}
	\mu^*&:= k_{ZZ} (\sigma^2 k_{ZZ} + k_{ZX}k_{XZ})^{-1} k_{ZX} y \label{eq:optimal-mu} \\
	\Sigma^*&:=k_{ZZ} ( k_{ZZ} + \sigma^{-2} k_{ZX}k_{XZ})^{-1}  k_{ZZ}  \label{eq:optimal-sigma}
\end{align}
and the resulting ELBO, denoted by $\mathcal{L}^*$, is 
\begin{align}
 \mathcal{L}^*  = & - \frac{1}{2} \log {\rm det} (q_{XX} + \sigma^2 I_n) - \frac{1}{2} y^\top (q_{XX} + \sigma^2 I)^{-1} y \nonumber \\
& - \frac{n}{2} \log 2 \pi - \frac{1}{2\sigma^2} {\rm tr} (k_{XX} - q_{XX}), \label{eq:ELBO-optimal}
\end{align}
where $q$ is the approximate kernel in \eqref{eq_approx_q}.
Inserting these expressions in the definition of the variational distribution \eqref{eq_VGP_par_fam}, the optimal variational approximation (for fixed inducing points $Z$) is given by $GP(m^*$, $k^*)$ with\footnote{\textcolor{black}{Note that these $m^*$ and $k^*$ are also equal to the predictive mean and covariance functions of the so-called Deterministic Training Conditional (DTC) approximation \citep{seeger2003fast} \citep[Section 5]{quinonero2005unifying}; see \citet[Sections 2.3 and 2.4]{bauer2020advances}. Therefore, RKHS interpretations for these functions also apply to the DTC approximation. However, the DTC's objective function for hyperparameter optimization is different from that of the SVGP, i.e., the ELBO in \eqref{eq:ELBO-optimal}. Therefore RKHS interpretations for the ELBO \eqref{eq:ELBO-optimal} is only valid for the SVGP.}}
\begin{align}
	m^*(x) &:= k_{Z}(x)^\top (\sigma^2 k_{ZZ} + k_{ZX} k_{XZ})^{-1} k_{ZX}^{} y \label{eq_optimal_m} \\
	k^*(x,x')&:= k(x,x') - k_Z(x)^\top k_{ZZ}^{-1} k_Z(x') \nonumber\\
	&+  k_Z(x)^\top  ( k_{ZZ} + \sigma^{-2} k_{ZX}k_{XZ})^{-1}  k_Z(x'). \label{eq_var_kernel}
\end{align}

The computational complexity of obtaining the mean function $m^*$ and the covariance function $k^*$ is $\mathcal{O}(nm^2 + m^3)$, which can be much smaller than the complexity $\mathcal{O}(n^3)$ of the exact posterior as long as the number of inducing points $m$ is much smaller than the training data size $n$.


The ELBO  \eqref{eq:ELBO-optimal} with optimal $\mu^*$ and $\Sigma^*$ is a key quantity, as it can be used  i) as a criterion for optimizing the inducing inputs $z_1, \dots, z_m$ and ii) for theoretically analyzing the quality of variational approximation.

\textcolor{black}{
We have seen that the Nystr\"om  and SVGP  approximations are derived from quite different principles; the former approximately solves KRR by replacing the RKHS by its subspace spanned by the inducing points, while the latter variationally approximates the GP posterior using variational GPs defined with the inducing points. 
How are these two approximations  related? This is the question we  investigate in the rest of the paper.
}

\section{Connections between the SVGP and Nystr\"om Approximations}\label{chapter_large_scale_approx}

\textcolor{black}{
This section studies how the Nystr\"om and SVGP approximations are related. We first analyze the equivalence between the Nystr\"om KRR and the optimized mean function of the SVGP in Section  \ref{sec:eq_opt_mean_cov}. We then introduce orthogonal projections on the RKHS subspace in Section \ref{sec:charac_approx_kernel}, and investigate RKHS interpretations of the SVGP covariance function in Section \ref{sec:var_mean_cov}.
}

\subsection{Equivalence between the Nyst\"om KRR and the SVGP Posterior Mean}\label{sec:eq_opt_mean_cov}

\textcolor{black}{
We first study how the Nystr\"om KRR \eqref{eq_nystroem} is related to the optimized mean function \eqref{eq_optimal_m} of the SVGP. 
Their equivalence immediately follows from their forms, as summarized below.}


\begin{theorem} \label{theo:equivalence-pred}
  Let $k$ be a kernel, and suppose that data $(x_i, y_i)_{i=1}^n \subset \X \times \mathbb{R}$ are given. 
  Let $Z = (z_1, \dots, z_m) \in \mathcal{X}^m$ be fixed inducing inputs such that the kernel matrix $k_{ZZ} = (k(z_i, z_j))_{i,j = 1}^m \in \mathbb{R}^{m \times m}$ is invertible. 
\begin{itemize}
    \item Let $m^*$ be the optimized mean function \eqref{eq_optimal_m} of the SVGP approximation using $Z$ as inducing points for defining the variational family \eqref{eq_VGP_par_fam}.
    \item Let $\bar{f}$ be the Nystr\"om approximation \eqref{eq_nystroem} using $Z$ as landmark points for defining the subspace $M =  {\rm span}(k(\cdot, z_1), \dots, k(\cdot, z_m))$. 
\end{itemize}
 Then we have $m^* = \bar{f}$, provided $\sigma^2 = n \lambda$ for the noise variance $\sigma^2$ in \eqref{eq:likelihood} and the regularization constant $\lambda$ in \eqref {eq:Nystrom-opt}.
\end{theorem}

\textcolor{black}{
The condition $\sigma^2 = n \lambda$ in Theorem \ref{theo:equivalence-pred} is the same as that required for the equivalence between KRR and GPR, as mentioned in Section \ref{sec:connections_exact}. }


\textcolor{black}{
Now the question is why the equivalence in Theorem \ref{theo:equivalence-pred} holds, even though the Nystr\"om and SVGP approximations are the solutions to seemingly unrelated optimization problems. 
To investigate this question, we closely inspect the ELBO  in \eqref{eq_def_L},   the objective function of the SVGP, as summarized in Theorem \ref{theo:ELBO-expression} below.
The proof can be found in Appendix \ref{sec:proof-elbo-interp}. }


\begin{theorem} \label{theo:ELBO-expression}
Let $\nu = (Z, \mu, \Sigma) \in \mathcal{X}^m \times \mathbb{R}^m \times \mathbb{R}^{m \times m}_{> 0}$ be such that the kernel matrix $k_{ZZ} \in \mathbb{R}^{m \times m}$ is invertible, and let $\mathcal L(\nu)$ be the ELBO in \eqref{eq_def_L}. 
Then we have
\begin{align}
 -2 \sigma^2  \mathcal L(\nu)  & =    \sum_{i=1}^n \left( y_i - k_{Z}(x_i)^\top k_{ZZ}^{-1}  \mu \right)^2 + \sigma^2  \mu k_{ZZ}^{-1} \mu \label{eq:kernel_inter_ap}  \\
& + \sum_{i=1}^n k_Z(x_i)^\top k_{ZZ}^{-1} \Sigma k_{ZZ}^{-1} k_{Z}(x_i) \label{eq:q_nu_term}  \\
& \quad + \sigma^2 \left( {\rm tr}(k_{ZZ}^{-1} \Sigma)   + \log \left( \frac{ {\rm det} (k_{ZZ}) }{ {\rm det} (\Sigma) }\right) - m \right) \label{eq:KL_term_ap} \\
&  + \sum_{i=1}^n \left(  k(x_i, x_i) - k_Z(x_i)^\top k_{ZZ}^{-1} k_Z(x_i) \right) \label{eq:residual_term_ap} . 
\end{align}

\end{theorem}


\textcolor{black}{
Theorem \ref{theo:ELBO-expression} disentangles the ELBO \eqref{eq_def_L} into i) the part \eqref{eq:kernel_inter_ap}  that depends only on the mean parameter $\mu$ and the inducing inputs $Z$, ii) the part \eqref{eq:q_nu_term} \eqref{eq:KL_term_ap} that depends only on the covariance parameter $\Sigma$ and $Z$, and iii) the part \eqref{eq:residual_term_ap} that depends only on $Z$.  
We focus here on the part i)  to understand the equivalence in Theorem \ref{theo:equivalence-pred}. The other parts are analyzed later. 
}

\begin{figure}
    \centering
    \includegraphics[width=0.7\linewidth]{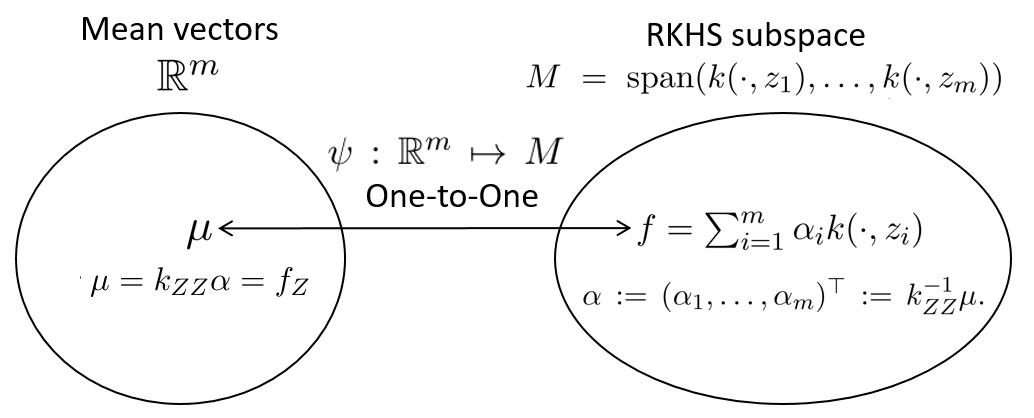}
    \caption{\textcolor{black}{Illustration of the relation between the SVGP mean vector $\mu \in \mathbb{R}^m$ and the corresponding function $f = \sum_{i=1}^m \alpha_i k(\cdot, z_i)$ in the RKHS subspace $M = {\rm span}( k(\cdot, z_1), \dots, k(\cdot, z_m) )$, where $\alpha := (\alpha_1, \dots, \alpha_m)^\top := k_{ZZ}^{-1} \mu$. } }
    \label{fig:mean-vec-subspace}
\end{figure}

\textcolor{black}{
As a preliminary, notice that the SVGP mean function  $m^{\nu} = k_Z(\cdot)^\top k_{ZZ}^{-1} \mu$  in \eqref{eq:mean-variational} lies in the subspace $M = {\rm span}( k(\cdot, z_1), \dots, k(\cdot, z_m) )$, because it can be written $m^\nu = \sum_{i=1}^m \alpha_i k(\cdot, z_i)$ for $\alpha := (\alpha_1, \dots, \alpha_m)^\top :=  k_{ZZ}^{-1} \mu$.  
Define then a map $\psi : \mathbb{R}^m \mapsto M$ that maps the mean vector $\mu \in \mathbb{R}^m$ to the mean function $m^\nu$:
\begin{equation} \label{eq:injective-map}
\psi(\mu) := k_Z(\cdot)^\top k_{ZZ}^{-1} \mu \in M, \quad \mu \in \mathbb{R}^m.
\end{equation}
This map is {\em one-to-one}, because  each $f \in M$ can be written $f = \sum_{i=1}^m \alpha_i k(\cdot,z_i)$ for uniquely associated $\alpha := (\alpha_1, \dots, \alpha_m)^\top \in \mathbb{R}^m$ and we have $f = \psi(\mu)$ for   $\mu = k_{ZZ} \alpha = f_Z \in \mathbb{R}^m$, which is the evaluations of $f$ at the inducing points $Z = (z_1, \dots, z_m)$ (see Figure \ref{fig:mean-vec-subspace} for an illustration.).
Therefore,   the inverse map $\psi^{-1}: M \mapsto \mathbb{R}^m$  is given by:
$$
\psi^{-1} (f) = f_Z \in \mathbb{R}^m, \quad f \in M. 
$$
}

\textcolor{black}{
With this preparation, and letting   $f := \psi(\mu) \in M$, the part \eqref{eq:kernel_inter_ap} can be written   
\begin{align}
 & \sum_{i=1}^n \left( y_i - k_{Z}(x_i)^\top k_{ZZ}^{-1}  \mu \right)^2 + \sigma^2  \mu k_{ZZ}^{-1} \mu  
 = \sum_{i=1}^n \left( y_i - f(x_i )   \right)^2 + \sigma^2   \|f\|_{\mathcal{H}_k}^2. \label{eq:SVGP-Nystorm-1205}
\end{align}
The right expression of \eqref{eq:SVGP-Nystorm-1205} is exactly the objective function \eqref{eq:Nystrom-opt}  of the Nystr\"om KRR  if $\sigma^2 = n \lambda$.
Hence, the negative ELBO as an objective function of $\mu \in \mathbb{R}^m$ is equivalent to the Nystr\"om KRR  objective as a function  $f = \psi(\mu) \in M$.  Consequently, they lead to the identical solution  in Theorem \ref{theo:equivalence-pred}.    
These findings are summarized below. 
}

 
\textcolor{black}{
\begin{corollary}  \label{coro:equiv-predict}
Let $\nu = (Z, \mu, \Sigma) \in \mathcal{X}^m \times \mathbb{R}^m \times \mathbb{R}^{m \times m}_{\succ 0}$ be such that the kernel matrix $k_{ZZ} \in \mathbb{R}^{m \times m}$ is invertible, $M = {\rm span}(k(\cdot,z_1), \dots, k(\cdot, z_m)) \subset \mathcal{H}_k$ be the subspace, and let $\mathcal L(\nu) $ be the ELBO in \eqref{eq_def_L}.
Let $\psi:  \mathbb{R}^m \mapsto M $ be the one-to-one map in \eqref{eq:injective-map}. 
Then we have for $f = \psi(\mu) \in M$ 
$$
- 2 \sigma^2 \mathcal{L}(\nu) = \sum_{i=1}^n \left( y_i - f(x_i) \right)^2 + \sigma^2 \|  f \|_{\mathcal{H}_k}^2 + L(Z,\Sigma), 
$$
where  $L(Z,\Sigma) \in \mathbb{R}$ is a constant not depending on $\mu$. 
In particular, we have $\psi(\mu^*) = f^*$, where
\begin{align*}
    \mu^* := \arg\max_{\mu \in \mathbb{R}^m} \mathcal{L}(  \nu ), \quad f^* & := \argmin_{f \in M}  \sum_{i=1}^n (y_i -  f (x_i)  )^2 + \sigma^2 \| f \|_{\mathcal{H}_k}^2.
\end{align*}
\end{corollary} 
}

\textcolor{black}{
It has been shown that the equivalence between the optimized SVGP mean function and the Nystr\"om KRR  in Theorem \ref{theo:equivalence-pred} stems from the equivalence between their objective functions.    Specifically, the optimized SVGP mean function is interpreted as minimizing the KRR objective on the subspace $M = {\rm span}(k(\cdot, z_1), \cdots, k(\cdot, z_m))$.  
}

\textcolor{black}{
We now turn our attention to the SVGP covariance function \eqref{eq_var_kernel} and investigate its connection to the Nystr\"om approximation.  To this end, we need the notion of {\em orthogonal projections} onto the subspace $M$, which is defined next.
}

\subsection{Orthogonal Projections onto the Subspace}\label{sec:charac_approx_kernel}

\textcolor{black}{
As a preliminary for analyzing the SVGP covariance function, we introduce orthogonal projections onto the RKHS subspace $M =  {\rm span}(k(\cdot,z_1), \dots, k(\cdot, z_m))$. 
We show that such projections define an approximate kernel whose RKHS is identical to $M$, and the use of this kernel in KRR leads to the Nystr\"om KRR. 
These points are summarized here, as they will be useful later.
}

\begin{figure}
    \centering
    \includegraphics[width=0.6\linewidth]{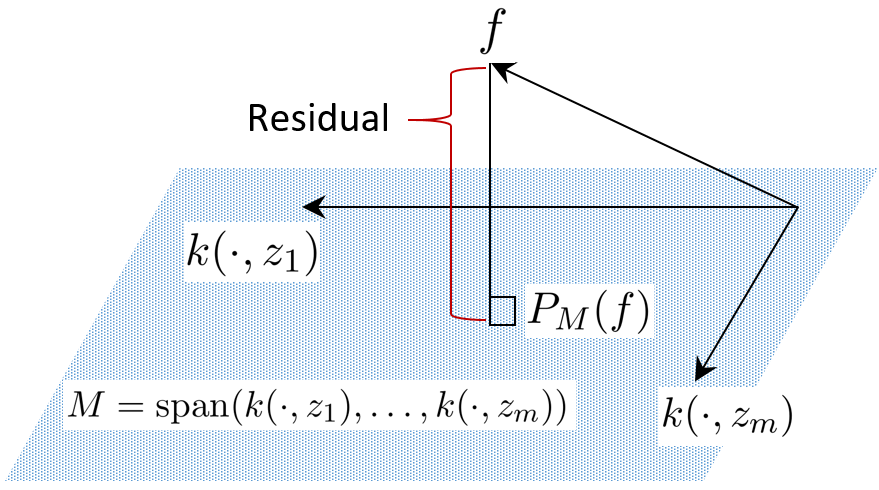}
    \caption{\textcolor{black}{Illustration of the orthogonal projection $P_M(f)$ of an RKHS element $f \in \mathcal{H}$ onto the subspace $M = {\rm span}(k(\cdot,z_1), \dots, k(\cdot,z_m))$.}}
    \label{fig:projection}
\end{figure}

\textcolor{black}{
Define $P_M: \mathcal{H}_k \mapsto M$ as the orthogonal projection operator that maps a given $f \in \mathcal{H}_k$ to its best approximation $P_M(f)$ in the subspace   $M$. 
That is, $P_M(f)$ is the element in $M$  having the minimum RKHS distance to $f$ (see Figure \ref{fig:projection}):
\begin{align*}
\left\| f - P_M(f) \right\|_{\mathcal{H}_k} =  \min_{g \in M} \left\| f - g \right\|_{\mathcal{H}_k} = \min_{\alpha_1,\dots,\alpha_m \in \mathbb{R}} \left\| f - \sum_{i=1}^m \alpha_i k(\cdot, z_i) \right\|_{\mathcal{H}_k}.  
\end{align*} 
If the feature maps $k(\cdot,z_1), \dots, k(\cdot, z_m)$ are linearly independent, which is equivalent to the kernel matrix $k_{ZZ} \in \mathbb{R}^{m \times m}$ being invertible, it can be shown that the minimizer $\alpha_1, \dots, \alpha_m$ for the right expression is unique and given by $\alpha := (\alpha_1, \dots, \alpha_m)^\top = k_{ZZ}^{-1} f_Z$ with $f_Z = (f(z_1), \dots, f(z_m))^\top$.\footnote{\textcolor{black}{By the reproducing property, we have $\left\| f - \sum_{i=1}^m \alpha_i k(\cdot, z_i) \right\|_{\mathcal{H}_k}^2 = \| f \|_{\mathcal{H}_k}^2 - 2f_Z^\top k_{ZZ} \alpha + \alpha^\top k_{ZZ} \alpha$, where $\alpha = (\alpha_1, \dots, \alpha_n)$. Taking the derivative with respect to $\alpha$ and equating it to $0$  yields the expression.}}  
Thus the orthogonal projection is given as
\begin{equation} \label{eq:projection-form}
	P_M\big(f\big) = k_Z(\cdot)^\top k_{ZZ}^{-1} f_Z.
\end{equation}
}


One can use the orthogonal projection operator $P_M$ to define an approximate kernel. Note that the kernel $k(x,x')$ can be written as the inner product between the feature maps $k(\cdot, x)$ and $k(\cdot, x')$:
\begin{equation*}
	k(x,x')=  \big\langle k(\cdot, x), k(\cdot, x') \big\rangle_{\mathcal{H}_k}, \ \ \ x,x' \in \X.
\end{equation*}
By replacing $k(\cdot, x)$ and $k(\cdot, x')$ by their best approximations $P_M(k(\cdot, x))$ and $P_M(k(\cdot, x'))$, we can define an approximation $q(x,x')$ of $k (x,x')$ as
\begin{align}
	q(x,x') &:= \big\langle P_M \big( k(\cdot, x) \big), P_M \big( k(\cdot, x') \big) \big\rangle_{\mathcal{H}_k} \nonumber \\
	&= \big\langle  k_{Z}(x)^\top k_{ZZ}^{-1} k_{Z}(\cdot) ,
	 k_{Z}(x')^\top k_{ZZ}^{-1} k_{Z}(\cdot)
 \big\rangle_{\mathcal{H}_k} \nonumber
	\\
	&= k_{Z}(x)^\top k_{ZZ}^{-1} k_{Z}(x'), \ \ \ x,x' \in \X  \label{eq_approx_q},
\end{align}
where the second line follows from \eqref{eq:projection-form}, and the third uses the reproducing property.

Since $q(x,x')$ in \eqref{eq_approx_q} is a positive definite kernel, it induces its own RKHS $\mathcal{H}_q$. 
As the following lemma shows, this RKHS  $\mathcal{H}_q$  is nothing but the subspace $M$, and the inner product $\left<\cdot, \cdot, \right>_{\mathcal{H}_q}$ of $\mathcal{H}_q$ is the same as that of the original RKHS $\mathcal{H}_k$. 
The proof can be found in Appendix \ref{ap_proof_Hq=M}.

\begin{lemma}\label{lemma_M=H_q}
Let $Z = (z_1,\dots,z_m) \in \mathcal{X}^m$  be such that the kernel matrix $k_{ZZ}$ is invertible.
Let $M = {\rm span}(k(\cdot, z_1), \dots, k(\cdot, z_m))$ be the subspace of $\mathcal{H}_k$ in \eqref{eq_span_M}, and $\mathcal{H}_q$ be the RKHS of the approximate kernel $q$ in \eqref{eq_approx_q}. 
Then we have $M=\mathcal{H}_q$ as a set of functions, and 
\begin{align*}
    \langle f, g \rangle_{\mathcal{H}_q} = \langle f, g \rangle_{\mathcal{H}_k}, \quad \forall f,g \in M=\mathcal{H}_q .
\end{align*}
In particular,  we have
$$
\| f \|_{\mathcal{H}_q} = \| f \|_{\mathcal{H}_k}, \quad \forall f \in  M = \mathcal{H}_q .
$$
\end{lemma}

By using Lemma \ref{lemma_M=H_q}, the Nystr\"om KRR in \eqref{eq:Nystrom-opt} can be rewritten as KRR using the approximate kernel $q$ in \eqref{eq_approx_q} as summarized as follows.

\begin{theorem}
\label{thm_nystroem} 
	Let  $X:=(x_1,...,x_n) \in \mathcal{X}^n$ and $y:=(y_1,...,y_n)^\top \in \mathbb{R}^n$ be given.
 Let $Z = (z_1,\dots,z_m) \in \mathcal{X}^m$  be such that the kernel matrix $k_{ZZ}$ is invertible, and $\bar{f}$ be the Nystr\"om approximation $\bar{f}$ of KRR in \eqref{eq:Nystrom-opt}. Then we have
	\begin{equation*}
	\bar{f} = \underset{f \in \mathcal{H}_q}{\argmin } \ \frac{1}{n} \sum_{i=1}^{n} \big(y_i - f(x_i) \big)^2 + \lambda \|f \|^2_{\mathcal{H}_q}
	\end{equation*}
	where $q$ is the approximate kernel defined in  \eqref{eq_approx_q}.
\end{theorem}
Theorem \ref{thm_nystroem} shows that the Nystr\"om KRR can be interpreted as KRR using the approximate kernel  \eqref{eq_approx_q}.




\label{sec:connection-DTC}


\textcolor{black}{
Consider now GPR using the approximate kernel $q$ in \eqref{eq_approx_q}  as the prior, $F \sim GP(0, q)$, instead of the original kernel $k$ \citep[Section 4]{quinonero2005unifying}. 
The resulting posterior  mean function $\bar{m}(x)$ and covariance function ${\bar q}(x,x')$  are given by 
\begin{align}
	\bar{m}(x) &= q_X(x)^\top (q_{XX} + \sigma^2 I_n)^{-1} y \nonumber \\
	&= k_{Z}(x)^\top (\sigma^2 k_{ZZ} + k_{ZX} k_{XZ})^{-1} k_{ZX}^{} y, \label{eq:post-mean-DTC} \\
	\bar{q}(x,x') &=   q(x,x') - q_X(x) (q_{XX}+\sigma^2 I_n )^{-1} q_X(x')  \nonumber \\
	&= k_{Z}(x)^\top  ( k_{ZZ} + \sigma^{-2} k_{ZX}k_{XZ})^{-1}  k_{Z}(x'). \label{eq:post-cov-DTC}
\end{align}	
The posterior mean function \eqref{eq:post-mean-DTC} is identical to the the Nystr\"om approximation \eqref{eq_nystroem}, if $\sigma^2 = n \lambda$.   This equivalence follows from the equivalence between the KRR and GPR and Theorem \ref{thm_nystroem}.
}


\subsection{RKHS Interpretations of the SVGP Covariance Function}
\label{sec:var_mean_cov}

\textcolor{black}{
Given the above preliminaries, we now analyze the optimized SVGP covariance function \eqref{eq_var_kernel}, which is given as
\begin{equation} \label{eq:SVGP-opt-cov-1321}
k^*(x,x') = \underbrace{ k(x,x') - k_Z(x)^\top k_{ZZ}^{-1} k_Z(x') }_{(a)}  + \underbrace{ k_Z(x)^\top  ( k_{ZZ} + \sigma^{-2} k_{ZX}k_{XZ})^{-1}  k_Z(x') }_{(b)}. 
\end{equation}
The term $(a)$ is the posterior covariance function of GPR using the prior $F \sim GP(0,k)$ given noise-free observations at the inducing inputs $Z = (z_1, \dots, z_m)$, as can be seen from \eqref{eq_bar_k} with $X =  Z$ and $\sigma^2 = 0$.    
The term $(b)$ is the posterior covariance function  \eqref{eq:post-cov-DTC} of GPR using the prior $F \sim GP(0,q)$ with the approximate kernel in \eqref{eq_approx_q} given noisy observations at training inputs $X = (x_1, \dots, x_m)$. Therefore the optimized SVGP covariance function is the sum of two posterior covariance functions. 
}

\textcolor{black}{
The optimized SVGP covariance function is an approximation to the exact GP posterior covariance function $\bar{k}(x,x')$ in \eqref{eq_bar_k}.  
To understand how  this approximation is  done, consider the optimized SVGP {\em variance}, i.e., $x = x'$ in \eqref{eq:SVGP-opt-cov-1321}, which approximates the exact GP posterior variance $\bar{k}(x,x)$. 
Intuitively, the term $(a)$ represents the uncertainty arising from approximating the GP prior  $F \sim GP(0,k)$ using the approximate prior $F \sim GP(0,q)$, as it can be written as $k(x,x) - q(x,x)$.  The term $(b)$ represents the posterior uncertainty using the approximate prior $F \sim GP(0,q)$. 
}

\textcolor{black}{
We provide two RKHS interpretations of the SVGP posterior covariance, one  geometric and one based on a function-space viewpoint.  They supplement the understanding of how the SVGP covariance approximates the exact GP posterior covariance. 
}

\subsubsection{Geometric Interpretation}

\textcolor{black}{
We describe a geometric interpretation of the optimized SVGP covariance function \eqref{eq:SVGP-opt-cov-1321}.  To this end, we first analyze the SVGP covariance function $k^\nu$ in \eqref{eq:cov-variational} parametrized by $\Sigma \in  \mathbb{R}^{m \times m}_{\succ 0}$:
\begin{equation} \label{eq:SVGP-cov1465}
k^\nu (x,x') = \underbrace{k(x,x') - k_Z(x)^\top k_{ZZ}^{-1} k_Z(x')}_{(a')} + \underbrace{ k_Z(x)^\top k_{ZZ}^{-1} \Sigma k_{ZZ}^{-1} k_Z(x') }_{ (b') }
\end{equation}
}

\textcolor{black}{
The first term  $(a')$ in \eqref{eq:SVGP-cov1465} is the same as the first term $(a) $ in \eqref{eq:SVGP-opt-cov-1321}. It can be written as
$$
k(x,x') - q(x,x') = \left< k(\cdot, x) - P_M \big( k(\cdot,x) \big), k(\cdot,x') - P_M \big( k(\cdot,x') \big) \right>_{\mathcal{H}_k},
$$
where $q(x,x')$ is the approximate kernel in \eqref{eq_approx_q} and 
$P_M\left( k(\cdot, x)\right) \in M$ is the orthogonal projection of the feature map $k(\cdot, x)$ onto the subspace $M = {\rm span}(k(\cdot, z_1), \dots, k(\cdot, z_m))$; see Section \ref{sec:charac_approx_kernel}.
The difference $ k(\cdot, x) - P_M \big( k(\cdot,x) \big)$ is thus the {\em residual of the approximation} of $k(\cdot,x)$ by $ P_M \big( k(\cdot,x) \big)$ (imagine Figure \ref{fig:projection} with $f = k(\cdot,x)$).  Therefore, the term $(a')$ is the inner product between two such residuals. 
In particular, for $x  = x'$ the term $(a')$ is equal to the squared length of the residual of approximation:
$$
\left\| k(\cdot, x) - P_M \big( k(\cdot,x) \big)  \right\|^2_{\mathcal{H}_k}.
$$
This quantifies how closely $k(\cdot, x)$ can be approximated by a linear combination of the feature maps $k(\cdot, z_1), \dots, k(\cdot, z_m)$ of the inducing inputs $Z = (z_1, \dots, z_m)$.  
}


\textcolor{black}{
 The second term $(b')$ is interpreted as a kernel defined as 
\begin{equation} \label{eq:kernel-sigma1483}
k^\Sigma(x,x') :=  \left< \phi^\Sigma(x), \phi^\Sigma(x') \right>_{\mathbb{R}^m},
\end{equation}
where $\phi^\Sigma: \mathcal{X} \mapsto \mathbb{R}^m$ is a finite-dimensional feature map parametrized by $\Sigma$ and $Z$: 
\begin{equation} \label{eq:feature-vec}
\phi^\Sigma(x) := \Sigma^{1/2} k_{ZZ}^{-1} k_Z(x) \in \mathbb{R}^m. 
\end{equation}
The parameter matrix $\Sigma$ specifies the covariance matrix at the inducing variables $Z = (z_1, \dots, z_m)$: we have $k^\Sigma_{ZZ} = \Sigma$.  
}

\textcolor{black}{
The term $(b)$ of the optimized SVGP covariance in \eqref{eq:SVGP-opt-cov-1321} is the  kernel \eqref{eq:kernel-sigma1483}  with $\Sigma = \Sigma^*$  being the optimal covariance matrix  \eqref{eq:optimal-sigma} that maximizes the ELBO. 
By Theorem \ref{theo:ELBO-expression}, this optimal $\Sigma^*$ is the minimizer of the sum of   \eqref{eq:q_nu_term} and \eqref{eq:KL_term_ap}, which can be written as
\begin{align} \label{eq:KL-optimization}
 \sum_{i=1}^n k^\Sigma(x_i, x_i) + \sigma^2 KL( \mathcal{N}(0, \Sigma)\ \| \ \mathcal{N}(0, k_{ZZ}) ),
\end{align}
As an objective function of $\Sigma$, this is interpreted as a regularized empirical risk.
The first term, interpreted as an empirical risk, is the sum of uncertainties (variances) $k^\Sigma(x_i,x_i)$ of the latent function at the training inputs $x_1, \dots, x_n$; lower uncertainties are encouraged because observations $y_1, \dots, y_n$ there are given. 
The second term is a regularizer that encourages the covariance matrix $k_{ZZ}^\Sigma = \Sigma$ at the inducing inputs $Z = (z_1, \dots, z_m)$ to be similar to the covariance matrix $k_{ZZ}$ of the prior GP for the latent function. 
If the noise variance $\sigma^2$ is smaller (larger), the observations $y_1,\dots,y_n$ contain more (less) information about the latent function, and thus $\Sigma$ should make the uncertainties $k^\Sigma(x_i,x_i)$ smaller (larger). 
}

\textcolor{black}{
Dividing \eqref{eq:KL-optimization} by $n$, defining $\lambda = \sigma^2 /n$, and using the feature map \eqref{eq:feature-vec}, the objective function of $\Sigma$ can be written as
\begin{align} \label{eq:KL-optimization-length}
\frac{1}{n}   \sum_{i=1}^n \left\| \phi^\Sigma (x_i)  \right\|_{\mathbb{R}^m}^2 + \lambda KL( \mathcal{N}(0, \Sigma)\ \| \ \mathcal{N}(0, k_{ZZ}) ).
\end{align} Geometrically, $\Sigma$ is optimized to minimize the average lengths of the feature vectors $\phi^\Sigma(x_i)$ at the training input points $x_1, \dots, x_n$, with a constraint to that $\Sigma$ does not too deviate from $k_{ZZ}$.  In the limit $n \to \infty$ and for fixed $Z$,  the second term vanishes (since $\lambda = \sigma^2/n$), while the first term converges to  $$
\int \left\| \phi^\Sigma (x)  \right\|_{\mathbb{R}^m}^2 dP(x),
$$ assuming that  training inputs $x_1, \dots, x_n$ are i.i.d.~with a probability distribution $P$.  The optimal $\Sigma^*$ in this limit thus minimizes this population average of the length of the feature vector $\phi^\Sigma(x)$.  
}

\subsubsection{Worst Case Error Interpretation}

\textcolor{black}{
We next provide a {\em worst case error} interpretation of  the optimized SVGP variance function.
In general, the posterior  variance function of GPR can be written as a  worst case error of KRR in the unit ball of the RKHS \citep[Section 3.4]{kanagawa2018gaussian}. 
Therefore, the optimized SVGP variance function, $k^*(x,x)$, can be written as the sum of two worst case errors: one is that of KRR with noise-free observations at  $Z = (z_1, \dots, z_m)$, and the other is that of Nystr\"om KRR given training observations at $X = (x_1, \dots, x_m)$. Formally, the following RKHS interpretation holds for the optimized SVGP variance function.
}

\begin{theorem} \label{theo:RKHS-interp-variational-cov}
Suppose that data $(x_i, y_i)_{i=1}^n \subset \X \times \mathbb{R}$ are given, and that $Z = (z_1, \dots, z_m) \in \mathcal{X}^m$ are fixed inducing inputs such that the kernel matrix $k_{ZZ} = (k(z_i, z_j))_{i,j = 1}^m \in \mathbb{R}^{m \times m}$ is invertible. 
 For the approximate kernel $q$ in \eqref{eq_approx_q}, define $q^\sigma$ as  
 \begin{align}  \label{eq:aug-approx-kernel}
    q^\sigma(x,x') = q(x,x') + \sigma^2 \delta(x, x'),  
\end{align}
where $\delta(x, x') = 1$ if $x = x'$ and $\delta(x,x') =0 $ otherwise, and let  $\mathcal{H}_{q^\sigma}$ be the RKHS of $q^\sigma$. 
  Let $k^*$ be the optimized SVGP  covariance function \eqref{eq_var_kernel}.
Then, for $x \not \in \{x_1,\dots,x_n \}$, we have
\begin{align}
   k^*(x,x) + \sigma^2 & =  \Big( \sup_{ \| f \|_{\mathcal{H}_k} \leq  1}  \{ f(x) -  \underbrace{k_Z(x)^\top k_{ZZ}^{-1} f_Z}_{\text{\rm Kernel Interpolation} }  \} \Big)^2 \label{eq:SVGP-var-worst1} \\
 & + \Big( \sup_{ \| h \|_{\mathcal{H}_{q^\sigma}} \leq 1} \{ h(x) - \underbrace{ k_{Z}(x)^\top (\sigma^2 k_{ZZ} + k_{ZX} k_{XZ})^{-1} k_{ZX} h_X }_{\text{\rm Nystr\"om KRR}} \}  \Big)^2. \label{eq:SVGP-var-worst2}
\end{align}
\end{theorem}

\begin{proof}
\textcolor{black}{
The assertion follows from  Theorem \ref{thm_nystroem}, \eqref{eq:post-cov-DTC}, and   \citet[Propositions 3.8 and 3.10]{kanagawa2018gaussian}.
} 
\end{proof}

\textcolor{black}{
Theorem \ref{theo:RKHS-interp-variational-cov}  shows that the optimized SVGP variance  $k^*(x,x)$ plus the noise variance $\sigma^2$ is equal to the sum of two terms in \eqref{eq:SVGP-var-worst1}  and \eqref{eq:SVGP-var-worst2}. 
The first term \eqref{eq:SVGP-var-worst1} is the worst case error of kernel-based interpolation $k_Z(x)^\top k_{ZZ}^{-1} f_Z$, where each $f \in \mathcal{H}_k$ with $\| f \|_{\mathcal{H}_k} \leq 1$ represents an unknown ground-truth whose output $f(x)$ at test input $x$ is to be interpolated from noise-free observations $f_Z = (f(z_1), \dots, f(z_m))^\top$ at the inducing inputs  $Z = (z_1, \dots, z_m)$. 
}

\textcolor{black}{
The second term \eqref{eq:SVGP-var-worst2} is the worst case error of the Nystr\"om KRR prediction (see  \eqref{eq_nystroem}) 
$$k_{Z}(x)^\top (\sigma^2 k_{ZZ} + k_{ZX} k_{XZ})^{-1} k_{ZX} h_X$$   of the output $h(x)$ based on observations $h_X = (h(x_1), \dots, h(x_n))^\top$ at training inputs $X = (x_1, \dots, x_n)$, where each $h \in \mathcal{H}_{q^\sigma}$ with $\left\| h \right\|_{\mathcal{H}_{q^\sigma}} \leq 1$ represents an unknown ground-truth function  plus a noise function: it can be written as $h = f + \xi$ for some $f \in \mathcal{H}_q$ and $\xi  \in \mathcal{H}_{\sigma^2 \delta}$, where $\mathcal{H}_{\sigma^2 \delta}$ is the RKHS of the Kronecker delta kernel  $\sigma^2 \delta (x,x')  $ and $\xi$ is interpreted as a noise function; see \citet[Section 3.4]{kanagawa2018gaussian}; the $\sigma^2$ in the left hand side of \eqref{eq:SVGP-var-worst1} corresponds to the prediction error of this noise component. 
}



\textcolor{black}{
We have discussed the equivalences and connections between the Nystr\"om and SVGP approximations. In the next section, we focus on their theoretical properties, focusing on the quality of approximation. 
}

\section{Connections in Theoretical Properties of Sparse Approximations}\label{chapter_convergence_bounds}

We investigate here connections between the theoretical properties of the Nystr\"om and SVGP approximations.  The Nystr\"om method provides an approximation to the exact KRR solution, and the SVGP approximates the exact GP posterior. The quality of approximation of either approach depends on the choice of inducing inputs $Z = (z_1, \dots, z_m)$. We focus here on theoretical error bounds for the approximation quality of either approach, and investigate how they are related.

For the Nystr\"om approximation, researchers have studied various approaches for subsampling inducing inputs $z_1, \dots, z_m$ from training inputs $x_1, \dots, x_n$ and their theoretical properties. These range from uniform subsampling to subsampling methods based on leverage scores \citep{rudi2015less,musco2016recursive,chen2021fast}, determinantal point processes (DPPs) \citep{li2016fast}, and to ensemble methods \citep{kumar2009ensemble,Kumar12Sampling}. Theoretical works either quantify a (relative) deviation of the approximate kernel matrix from the exact one and its impact on downstream tasks  \cite[e.g.,][]{cortes2010impact,musco2016recursive}, or more directly bound the expected loss of the resulting approximate KRR estimator \cite[e.g.][]{bach2013sharp,alaoui2015fast,rudi2015less}. 
On the other hand, for the SVGP approach, \citet{burt2019rates,BurtJMLR20} provided the first theoretical analysis of its approximation quality. 
\textcolor{black}{
Subsequently, \citet{Nieman22Contraction} extended this analysis to the Frequentist analysis of the contraction of the SVGP posterior to the regression function. \citet{vakili22Improved} conducted an asymptotic analysis of confidence intervals obtained from the SVGP approximation. 
}


\textcolor{black}{
In Section \ref{sec:burt-lemma}, we first discuss a fundamental bound of \citet{BurtJMLR20} on the KL divergence between the approximate and exact GP posteriors.
Section \ref{sec:theory-connection-Nystrom} then discusses a connection between the KL divergence for the SVGP approximation and the excess risk of the Nystr\"om KRR over the exact KRR.  Specifically, we show how an upper bound on the KL divergence leads to an upper bound on the excess risk.  Section \ref{sec:RKHS-error-bound} establishes an upper bound on the approximation error of the Nystr\"om KRR as measured by the RKHS distance. We then describe how this bound leads to upper bounds on the approximation errors of the SVGP posterior mean function and its derivatives. Lastly, Section \ref{sec:lower-bound-approx} describes how a lower bound for the approximation error of the SVGP leads to a lower bound for the approximation error of the Nystr\"om KRR. 
}

\subsection{A Fundamental Result of Burt et al. (2020)} \label{sec:burt-lemma}

We first consider the approximation quality of the SVGP approach.
In particular, we study a fundamental result of Burt (Lemma 3), from which many other results in \citet{BurtJMLR20} are derived. 
As before, let $\mu^*$ and $\Sigma^*$ be the optimal variational parameters in \eqref{eq:optimal-mu} and \eqref{eq:optimal-sigma}, respectively, and let $Z = (z_1, \dots, z_m) \in \mathcal{X}^m$ be $m$ inducing inputs such that the kernel matrix $k_{ZZ}$ is invertible. 
Let $\nu^* := (Z, \mu^*, \Sigma^*)$ and $\Q^{\nu^*} = GP(m^*, k^*)$ be the resulting variational GP posterior with mean function $m^*$ and covariance function $k^*$ in \eqref{eq:mean-variational} and \eqref{eq:cov-variational}, respectively.

A natural metric of quantifying the approximation quality of $\Q^{\nu^*}$ is the KL divergence to the exact GP posterior $\P^{F|y}$, which is given by \eqref{eq_VGPR_ELBO} with $\nu^* = (Z, \mu^*, \Sigma^*)$
\begin{align*}
        &KL \big(\Q^{\nu^*} \ \| \ \P^{F|y} \big) 
        = \log p(y) - \mathcal{L}^*,
\end{align*}
where $p(y)$ is the marginal likelihood and $\mathcal{L}^*$ is the ELBO in \eqref{eq:ELBO-optimal}. 
We know that  \cite[e.g.,][Eq.~(5.8)]{RasmussenWilliams}
\begin{align*}
    \log p(y) & = - \frac{1}{2} \log \det (k_{XX} + \sigma^2 I_n)  - \frac{1}{2} y^\top (k_{XX} + \sigma^2 I_n)^{-1} y - \frac{n}{2} \log 2 \pi 
\end{align*} 
Therefore,
\begin{align}
        &2 KL \big(\Q^{\nu^*} \ \| \ \P^{F|y} \big) 
        = 2 \log p(y) - 2 \mathcal{L}^* \nonumber \\
        & =  - \log \det (k_{XX} + \sigma^2 I_n) + \log \det (q_{XX} + \sigma^2 I_n)  \label{eq:KL-expression}  \\
        &  \quad -  y^\top (k_{XX} + \sigma^2 I_n)^{-1} y+  y^\top (q_{XX} + \sigma^2 I_n)^{-1} y  + \frac{1}{\sigma^2} {\rm tr}(k_{XX} - q_{XX}) \nonumber \\
        & \leq   -  y^\top (k_{XX} + \sigma^2 I_n)^{-1} y +  y^\top (q_{XX} + \sigma^2 I_n)^{-1} y \label{eq:KL-bound-data-terms}  + \frac{1}{\sigma^2} {\rm tr}(k_{XX} - q_{XX}), 
\end{align} 
where the inequality follows from $k_{XX} - q_{XX}$ being positive semi-definite.

\citet[Proof of Lemma 3]{BurtJMLR20} proceed to bound the first two terms in \eqref{eq:KL-bound-data-terms} as
\begin{align}
&  - y^\top (k_{XX} + \sigma^2 I_n)^{-1} y +  y^\top (q_{XX} + \sigma^2 I_n)^{-1} y \nonumber \\
& \leq \frac{\| y \|^2 \| k_{XX} - q_{XX} \|_{\rm op}}{\sigma^2 (\| k_{XX} - q_{XX} \|_{\rm op} + \sigma^2 )}  \leq \frac{\| y \|^2 {\rm tr}(k_{XX} - q_{XX} ) }{\sigma^2 ( {\rm tr} (k_{XX} - q_{XX} ) + \sigma^2 )}, \label{eq:bound-difference}
\end{align}
where \textcolor{black}{$\| A \|_{\rm op} := \sup_{ v \in \mathbb{R}^d,  \| v \| \leq 1 } \leq \| A v \|$ denotes the operator norm for any $A \in \R^{n \times n}$.}
Thus, we arrive at the following bound \citep[Lemma 3]{BurtJMLR20}: 
\begin{align}
        2 KL \big(\Q^{\nu^*} \ \| \ \P^{F|y} \big) & \leq \frac{ {\rm tr}(k_{XX} - q_{XX})}{\sigma^2} \left(  \frac{\| y \|^2   }{  {\rm tr} (k_{XX} - q_{XX} ) + \sigma^2 }   + 1 \right).  \nonumber \\
        & \leq \frac{ {\rm tr}(k_{XX} - q_{XX})}{\sigma^2} \left(  \frac{\| y \|^2   }{ \sigma^2 }   + 1 \right).  \label{eq:KL-bound} 
\end{align}

This result shows that the KL divergence becomes small if ${\rm tr}(k_{XX} - q_{XX})$ is small. The latter quantity becomes small if, intuitively, the approximate kernel matrix $q_{XX} = k_{XZ}k_{ZZ}^{-1}k_{ZX}$ is close to the exact one $k_{XX}$. 
 
\citet{BurtJMLR20} then establish various results on the KL divergence for the SVGP approximation by i) relating ${\rm tr}(k_{XX} - q_{XX})$ to the eigenvalues of $k_{XX}$ by considering a specific sampling scheme for $Z$, such as DPPs and leverage score sampling, ii) relating the eigenvalues of $k_{XX}$ to those of the corresponding kernel integral operator, and iii) bounding the eigenvalue decays of the integral operator by considering specific choices of the kernel $k$ and the probability distribution of training input points $x_1, \dots, x_n$.

\subsection{Connection to the Nystr\"om KRR}
\label{sec:theory-connection-Nystrom}

We now investigate how the bounds on the KL divergence for the SVGP approach are related to the Nystr\"om KRR. The key is the following lemma, which provides an RKHS interpretation of $y^\top (k_{XX} + \sigma^2 I)^{-1} y$ (and $y^\top (q_{XX} + \sigma^2 I)^{-1} y$) appearing in the bound \eqref{eq:KL-bound-data-terms}.
\begin{lemma} \label{lemma:emp-risk-expression}
Let $k$ be a kernel with RKHS $\mathcal{H}_k$.
Let $X = (x_1, \dots, x_n) \in \mathcal{X}^n$ and $y = (y_1, \dots, y_n)^\top \in \mathbb{R}^n$ be given. 
 Then for any $\sigma^2 > 0$, we have
\begin{align*}
   & y^\top (k_{XX} + \sigma^2 I_n)^{-1} y   = \min_{f \in \mathcal{H}_k} \frac{1}{ \sigma^2 } \sum_{i=1}^n (y_i - f(x_i))^2 +  \| f \|_{\mathcal{H}_k}^2.
\end{align*}    
\end{lemma}
\begin{proof}
The proof can be found in Appendix \ref{sec:proof-lemma-emp-risk-expression}.
\end{proof}

\textcolor{black}{
Lemma \ref{lemma:emp-risk-expression} shows that $y^\top (k_{XX} + \sigma^2 I_n)^{-1}y$ is  the minimum of a scaled version of  the KRR objective function \eqref{eq_krr} with $\sigma^2 = n \lambda$. Since the minimum is attained by the KRR solution $\hat{f} = k_X(\cdot)^\top (k_{XX} + \sigma^2 I_n)^{-1} y$, Lemma \ref{lemma:emp-risk-expression} thus implies that, for $\sigma^2 = n \lambda$,  
\begin{align*}
& y^\top (k_{XX} + \sigma^2 I_n)^{-1} y  = \frac{1}{\sigma^2} \sum_{i=1}^n (y_i - \hat{f}(x_i))^2 +  \| \hat{f} \|_{\mathcal{H}_k}^2 = n R_n (\hat{f}; y),  
\end{align*}
where we define $R_n (f; y)$ for any $f:\mathcal{X} \mapsto \mathbb{R}$  as the regularized empirical risk of KRR:
$$
R_n(f; y) := \frac{1}{n} \sum_{i=1}^n (y_i  - f(x_i))^2 + \lambda \| f \|_{\mathcal{H}_k}^2
$$
Similarly, Lemma \ref{lemma:emp-risk-expression},  Theorem \ref{thm_nystroem} and Lemma \ref{lemma_M=H_q} imply that $y^\top (q_{XX} + \sigma^2 I_n) y$ is proportional to the KRR objective value of the Nystr\"om KRR solution $\bar{f} = k_{Z}(\cdot)^\top (\sigma^2 k_{ZZ} + k_{ZX} k_{XZ})^{-1} k_{ZX}y$:
\begin{align*}
    y^\top (q_{XX} + \sigma^2 I_n)^{-1} y  
  & = \min_{f \in \mathcal{H}_q} \frac{1}{\sigma^2} \sum_{i=1}^n (y_i  - \bar{f}(x_i))^2 + \| f \|_{\mathcal{H}_q}^2 \\
  & = \frac{1}{\sigma^2} \sum_{i=1}^n (y_i  - \bar{f}(x_i))^2 + \| \bar{f} \|_{\mathcal{H}_k}^2 = n R(\bar{f}; y).
\end{align*}
}

\textcolor{black}{
Therefore the quantity $y^\top (q_{XX} + \sigma^2 I_n)^{-1} y - y^\top (k_{XX} + \sigma^2 I_n)^{-1} y$ appearing in the KL divergence \eqref{eq:KL-expression} can be written as the difference between the KRR objectives for the Nystr\"om and exact solutions:
\begin{align}
&  y^\top (q_{XX} + \sigma^2 I_n)^{-1} y - y^\top (k_{XX} + \sigma^2 I_n)^{-1} y  \label{eq:difference-KRR-objective} \\
& =  n \left( R_n(\bar{f}, y) -  R_n(\hat{f}, y) \right) =  n\left( \min_{f \in M} R_n(f, y) - \min_{f \in \mathcal{H}_k} R_n(f; y) \right) \geq 0  \nonumber,
\end{align}
where the last identity follows from \eqref{eq_krr} and \eqref{eq:Nystrom-opt} and the inequality  from $M \subset \mathcal{H}_k$.
Thus, one can understand the difference \eqref{eq:difference-KRR-objective} as the ``excess risk'' of the Nystr\"om KRR estimator  $\bar{f}$ over the exact KRR estimator  $\hat{f}$. In this sense, it measures the approximation quality of the Nyst\"om KRR. 
 }



Now, combining the bounds in \eqref{eq:bound-difference} and the expression \eqref{eq:difference-KRR-objective}, we immediately have the corresponding bounds on the ``excess risk'' of the Nystr\"om KRR.  

\begin{corollary} \label{coro:bound-difference}
Let $k$ be a kernel with RKHS $\mathcal{H}_k$.
Let $X = (x_1, \dots, x_n) \in \mathcal{X}^n$ and $y = (y_1, \dots, y_n)^\top \in \mathbb{R}^n$ be given, and let $Z = (z_1, \dots, z_m) \in \mathcal{X}^m$ be such that the kernel matrix $k_{ZZ} \in \mathbb{R}^{m \times m}$ is invertible. 
  Let $\bar{f}$ and $\hat{f}$ be the Nystr\"om and exact KRR estimators in \eqref{eq_nystroem} and \eqref{eq:KRR-compact}, respectively, with a regularization constant $\lambda > 0$.
  Then we have
    \begin{align*}
  R_n(\bar{f}; y) - R_n(\hat{f}; y) 
 & \leq \frac{\| y \|^2 \| k_{XX} - q_{XX} \|_{\rm op}}{n^2 \lambda (\| k_{XX} - q_{XX} \|_{\rm op} + n \lambda )}   \leq \frac{\| y \|^2 {\rm tr}(k_{XX} - q_{XX} ) }{n^2 \lambda ( {\rm tr} (k_{XX} - q_{XX} ) +  n \lambda )},
\end{align*}
where $R_n(f; y) := \frac{1}{n} \sum_{i=1}^n (y_i  - f(x_i))^2 + \lambda \| f \|_{\mathcal{H}_k}^2$.
\end{corollary}

\textcolor{black}{
\begin{proof} 
Using  \eqref{eq:bound-difference} and \eqref{eq:difference-KRR-objective} and setting $\sigma^2 = n \lambda$, we have
\begin{align*}
 n \left( R_n(\bar{f}, y) -  R_n(\hat{f}, y) \right)  = &  - y^\top (k_{XX} + n \lambda I_n)^{-1} y +  y^\top (q_{XX} + n \lambda I_n)^{-1} y \nonumber \\
& \leq \frac{\| y \|^2 \| k_{XX} - q_{XX} \|_{\rm op}}{n \lambda(\| k_{XX} - q_{XX} \|_{\rm op} + n \lambda )}  \leq \frac{\| y \|^2 {\rm tr}(k_{XX} - q_{XX} ) }{n \lambda ( {\rm tr} (k_{XX} - q_{XX} ) + n\lambda )}. 
\end{align*}
The assertion follows immediately.
\end{proof}
}

\textcolor{black}{
We have shown that the KL divergence \eqref{eq:KL-expression} for the SVGP approximation contains the excess risk of the Nystr\"om KRR \eqref{eq:difference-KRR-objective}. Using this connection and bounds for the KL divergence from \citet{BurtJMLR20}, we have derived corresponding bounds for the excess risk of the Nystr\"om KRR. This demonstrates how one can translate results on GPs to RKHS-based methods. 
}


On the other hand, since we now know that there exist KRR objective functions in the expression \eqref{eq:KL-expression} of the KL divergence, it may be possible to use more sophisticated theoretical arguments for the Nystr\"om KRR \cite[e.g.][]{bach2013sharp,alaoui2015fast,rudi2015less,chen2021fast} to obtain sharper bounds on the KL divergence for the SVGP approximation. This investigation is reserved for future research.

\subsection{RKHS Error Bound for Nystr\"om KRR and its Implications to SVGP} 
\label{sec:RKHS-error-bound}

We present here an upper bound on the RKHS distance between the Nystr\"om and exact KRR estimators, which is novel to the best of our knowledge. 
We will apply this bound to obtain error bounds for the SVGP posterior mean function and its {\em derivatives}. 
The bound is summarized below, whose proof can be found in Appendix \ref{sec:Proof-nystrom-RKHS-error}.

\begin{theorem} \label{thm_a_posteriori_bound} \label{theo:nystrom-RKHR-error}
Let $k$ be a kernel with RKHS $\mathcal{H}_k$.
Let $X = (x_1, \dots, x_n) \in \mathcal{X}^n$ and $y = (y_1, \dots, y_n)^\top \in \mathbb{R}^n$ be given, and let $Z = (z_1, \dots, z_m) \in \mathcal{X}^m$ be such that the kernel matrix $k_{ZZ} \in \mathbb{R}^{m \times m}$ is invertible. 
  Let $\bar{f}$ and $\hat{f}$ be the Nystr\"om and exact KRR estimators in \eqref{eq_nystroem} and \eqref{eq:KRR-compact}, respectively, with a regularization constant $\lambda > 0$.
  Then we have
\begin{equation*}
    \| \hat{f} - \bar{f} \|_{\mathcal{H}_k}^2 \le \frac{2~{\rm tr}(k_{XX}-q_{XX}) \|y\|^2}{(n \lambda)^2} .
\end{equation*}
\end{theorem}

The upper-bound takes a similar form as the bound \eqref{eq:KL-bound} on KL divergence for the SVGP approximation in terms of the dependence on $ {\rm tr}(k_{XX}-q_{XX})$, $\| y \|^2$ and $\sigma^2 = n \lambda$. 

By the equivalence between the KRR estimator $\hat{f}$ and the GP posterior mean function $m^*$, and that of the Nystr\"om KRR estimator $\bar{f}$ and the SVGP posterior mean function $\bar{m}$ in Theorem  \ref{theo:equivalence-pred}, Theorem \ref{thm_a_posteriori_bound} directly leads to the corresponding bound for the SVGP posterior mean function, as summarize as follows.

\begin{corollary} \label{coro:SVGP-mean}
Let $k$ be a kernel with RKHS $\mathcal{H}_k$.
Let $X = (x_1, \dots, x_n) \in \mathcal{X}^n$ and $y = (y_1, \dots, y_n)^\top \in \mathbb{R}^n$ be given, and let $Z = (z_1, \dots, z_m) \in \mathcal{X}^m$ be such that the kernel matrix $k_{ZZ} \in \mathbb{R}^{m \times m}$ is invertible. 
Let $m^*$ and $\bar{m}$ be the SVGP and exact GP posterior mean functions in \eqref{eq_optimal_m} and \eqref{eq_bar_m}, respectively, with  prior $F \sim GP(0,k)$ and likelihood model \eqref{eq:likelihood} with noise variance $\sigma^2 > 0$. 
Then we have
\begin{align*}
    \| \bar{m}  - m^*\|_{\mathcal{H}_k}^2 \le \frac{2~{\rm tr}(k_{XX}-q_{XX}) \|y\|^2}{\sigma^4} .    
\end{align*}    
\end{corollary}

Note that the RKHS distance is stronger than the supremum norm between two functions. 
In fact, by the reproducing property, it can be shown that
\begin{align*}
( \bar{f}(x) - \hat{f}(x) )^2 \leq \| \bar{f} - \hat{f} \|_{\mathcal{H}_k}^2  k(x,x), \quad \forall x \in \mathcal{X}.
\end{align*}
Moreover, if the kernel $k$ is smooth, then the RKHS distance upper-bounds the derivatives of the RKHS functions. To describe this, let $\mathcal{X} \subset \mathbb{R}^d$ be an open set. Suppose that the kernel $k$ is continuously differentiable\footnote{Many commonly used kernels, such as the Gaussian kernel, satisfy this requirement. } on $\mathcal{X}$ in the sense that, for any $j = 1, \dots, d$, the partial derivative $\partial_j \partial'_j  k(x, x')$ exists and is continuous on $\mathcal{X}$, where $\partial_j$ and $\partial'_j $ denote the partial derivatives with respect to the $j$-th coordinate of the first and second arguments of $k(x,x')$, respectively. 
Then \citet[Corollary 4.36]{SteChr2008} implies that, for all $j=1,\dots,d$ and all $x \in \mathcal{X}$,
\begin{align*}
(  \partial_j \bar{f}(x) - \partial_j \hat{f}(x) )^2 \leq \| \bar{f} - \hat{f} \|_{\mathcal{H}_k}^2 \partial_j \partial'_j  k(x, x) ,  
\end{align*}
Thus, the bound in Theorem \ref{theo:nystrom-RKHR-error} implies that, if ${\rm tr}(k_{XX} - q_{XX})$ is small, then the partial derivatives (and thus the gradients) of the Nystr\"om KRR approximate well those of the exact KRR. 
By the same argument and Corollary \ref{coro:SVGP-mean}, we immediately obtain the following corollary on the equivalent result for the SVGP approximation.

\begin{corollary}
Suppose the same notation and assumptions in Corollary \ref{coro:SVGP-mean}. 
Let $\mathcal{X} \subset \mathbb{R}^d$ be an open set and assume that $k$ is continuously differentiable on $\mathcal{X}$. 
Then we have
for all $j=1,\dots,d$ and all $x \in \mathcal{X}$,
\begin{align*}
& ( \partial_j m^*(x) - \partial_j \bar{m}(x)  )^2  \leq  \frac{2~{\rm tr}(k_{XX}-q_{XX}) \|y\|^2 \partial_j \partial'_j  k(x, x)}{\sigma^4}.
\end{align*}
\end{corollary}

This shows that the SVGP can approximate not only the exact posterior mean function but also its derivatives, if ${\rm tr}(k_{XX}-q_{XX})$ is small. In applications where the derivative estimates are used (e.g., see \citealt{Jian17BO-gradients}), this result provides a support for using the SVGP approximation in place of the exact GP posterior means of derivatives.

\subsection{Lower Bounds for Approximation Errors} 
\label{sec:lower-bound-approx}

Lastly, we describe how lower bounds for the SVGP approximation lead to lower bounds for the Nystr\"om approximation. 
We discuss lower bounds for the average case errors of sparse approximations, by assuming a probabilistic model for training outputs $y = (y_1, \dots, y_n)^\top$. As before, we fix training inputs $X = (x_1, \dots, x_n)$ and inducing inputs $Z = (z_1, \dots, z_m)$. 
Following \citet{BurtJMLR20}, we consider the following model for $y$:
\begin{equation} \label{eq:model-y}
y | X \sim \mathcal{N}(0,k_{XX}+ \sigma^2 I_n)   
\end{equation}
which is given by the likelihood model \eqref{eq:likelihood} and by marginalizing the latent prior GP, $F \sim GP(0, k)$. 
\citet[Lemma 4]{BurtJMLR20} shows the following lower and upper bounds for the averaged KL divergence between the SVGP and exact GP posteriors:
	\begin{align}
	     \frac{ {\rm tr}(k_{XX}-q_{XX})}{2 \sigma^2}  & \leq  \E_y \left[ KL\big(\Q^{\nu^*} \ \| \ \P^{F|y}\big) \right] \label{eq:lower-bound}   \leq \frac{ {\rm tr}(k_{XX}-q_{XX})}{\sigma^2}, 
	\end{align}
where $\mathbb{E}_y$ denotes the expectation with respect to $y$ generated according to \eqref{eq:model-y}.

These lower and upper bounds are {\it a priori} bounds in the sense that they hold for the average with respect to the prior model and thus are informative before observing the actual training outputs $y_1, \dots, y_n$. While this performance measure (the averaged KL divergence) is less informative for the approximation accuracy after one has observed actual training outputs $y_1, \dots, y_n$ (the {\em a posteori} setting), the lower bound still provides a useful insight. Specifically, the lower bound \eqref{eq:lower-bound} is proportional to ${\rm tr}(k_{XX} - q_{XX})$. 
Thus if  ${\rm tr}(k_{XX} - q_{XX})$  is large, then the SVGP posterior $\Q^{\nu^*}$ {\em cannot} accurately approximate the exact posterior $\P^{F|y}$ on average. This is intuitively the case where the inducing inputs $Z = (z_1, \dots, z_m)$ do not effectively ``cover'' the training inputs $X = (x_1, \dots, x_n)$. 
Since ${\rm tr}(k_{XX} - q_{XX})$ appears both in the upper and lower bounds, the above result shows that ${\rm tr}(k_{XX} - q_{XX})$ can serve as an average performance metric for the SVGP approximation.

Now, based on the interpretation that the KL divergence \eqref{eq:KL-expression} contains the excess risk \eqref{eq:difference-KRR-objective} of the Nystr\"om KRR, the lower bound \eqref{eq:lower-bound} for the KL divergence leads to a lower bound for the excess risk of the Nystr\"om KRR. 

 \begin{corollary}  \label{eq:lower-bound-KRR}
 Let $k$ be a kernel with RKHS $\mathcal{H}_k$.
Let $X = (x_1, \dots, x_n) \in \mathcal{X}^n$ and let $Z = (z_1, \dots, z_m) \in \mathcal{X}^m$ be such that the kernel matrix $k_{ZZ} \in \mathbb{R}^{m \times m}$ is invertible. 
Suppose $y = (y_1, \dots, y_n)^\top \in \mathbb{R}^n$ are generated as \eqref{eq:model-y}.
  Let $\bar{f}$ and $\hat{f}$ be the Nystr\"om and exact KRR estimators in \eqref{eq_nystroem} and \eqref{eq:KRR-compact}, respectively, with a regularization constant $\lambda > 0$.
  Then we have
\begin{align*}
\frac{1}{n} \log \frac{\det (k_{XX} + n \lambda I_n)}{\det (q_{XX} + n \lambda I_n)}           
        & \leq \mathbb{E}_y\left[ R_n(\bar{f}; y) - R_n(\hat{f}; y) \right]
\end{align*}
where $R_n(f; y) := \frac{1}{n} \sum_{i=1}^n (y_i  - f(x_i))^2 + \lambda \| f \|_{\mathcal{H}_k}^2$.
\end{corollary} 

\textcolor{black}{
\begin{proof}
By \eqref{eq:KL-expression}, \eqref{eq:difference-KRR-objective} and \eqref{eq:lower-bound} with $\sigma^2 = n \lambda$, we have
\begin{align*}
&  \frac{ {\rm tr}(k_{XX}-q_{XX})}{n \lambda}   \leq  2 \E_y \left[ KL\big(\Q^{\nu^*} \ \| \ \P^{F|y}\big) \right]  \\
        & =  - \log \frac{  \det ( k_{XX} + n \lambda I_n) } { \det (q_{XX} + n \lambda I_n) } + \mathbb{E}_y \left[ n \left( R_n(\bar{f}, y) -  R_n(\hat{f}, y) \right)\right] + \frac{{\rm tr}(k_{XX} - q_{XX})}{n \lambda}.  \\
\end{align*}
The assertion immediately follows.
\end{proof}
}

In the left hand side of Corollary \ref{eq:lower-bound-KRR}, $\log {\rm det} (k_{XX} + n \lambda I)$ and $\log {\rm det} (q_{XX} + n \lambda I)$ can intuitively be interpreted as the complexities of the models associated with the kernels $k$ and $q$, respectively. Thus, Corollary shows that, if the complexity  for $q$ is much smaller than that for $k$, then the difference of the KRR objectives cannot be small on average. This suggests that the left hand side of Corollary \ref{eq:lower-bound-KRR} may be useful as a quality metric for the Nystr\"om approximation in the {\em a priori} setting.

\section{Conclusions} \label{sec:conclusion}
We have established various connections and equivalences between sparse approximation methods for GPR and KRR, namely the SVGP and Nystr\"om approximations. \textcolor{black}{We hope these contributions will help the two fields grow closer together and allow researchers to readily translate results from one field to another. 
In general, equivalent characterizations of the same problem enable one to look at the problem from a new angle and ultimately lead to new approaches. For example, the equivalent formulation of Bayesian posterior inference as an optimization problem \citep{csiszar1975divergence,donsker1975asymptotic} has led to a whole new class of inference algorithms 
\citep[e.g.,][]{jordan1999introduction,knoblauch2019generalized,khan2021bayesian}. Similarly, we hope our contributions will stimulate new research in both GP and RKHS communities by leveraging and extending the equivalences and connections in the paper.    
}   

\subsection*{Acknowledgements}
We express our gratitude to the Action Editor and the anonymous reviewers for their time and thoughtful comments that helped improve the paper. 
This work in part has been supported by the French Government, through the 3IA Cote d’Azur Investment in the Future Project managed by the National Research Agency (ANR) with the reference number ANR-19-P3IA-0002.

\renewcommand{\theHsection}{A\arabic{section}}

\appendix

{\color{black}

\section{
Derivations of the Optimal Variational Mean and Covariance Parameters
}\label{ap:alter_der}

We present derivations of the optimal variational mean $\mu^*$ and covariance matrix $\Sigma^*$ in \eqref{eq:optimal-mu} and \eqref{eq:optimal-sigma} based on the formulation of \citet{khan2021bayesian}.  To this end, let us rewrite  the ELBO in \eqref{eq_def_L} as 
    \begin{align}  
        \mathcal{L}(\nu) & = \E_{F^\nu \sim \Q^\nu} \big[ \log p(y|F^\nu_X) \big] - KL \big( \Q_Z^\nu \  \| \  \P_Z  \big)  \nonumber \\
        &  \stackrel{(A)}{=} \E_{F_Z^\nu \sim \Q_Z^\nu} \left[ \E_{F^\nu \sim  Q^\nu} \left[ \log p(y|F^\nu_X) \mid F_Z^\nu  \right]\right]  - KL \big( \Q_Z^\nu \  \| \  \P_Z  \big)  \nonumber \\
       &  \stackrel{(B)}{=} \E_{F_Z^\nu \sim \Q_Z^\nu} \left[ \E_{F^\nu \sim  Q^\nu} \left[ \log p(y|F^\nu_X) \mid F_Z^\nu  \right]\right]  -  \E_{F_Z^\nu \sim \Q_Z^\nu} \left[ \log  \Q_Z^\nu(F_Z^\nu) \right]  \nonumber  \\
       & = \E_{F_Z^\nu \sim \Q_Z^\nu} \left[ \E_{F^\nu \sim  Q^\nu} \left[ \log p(y|F^\nu_X) \mid F_Z^\nu  \right] + \log \P_Z( F_Z^\nu )   \right]  -  \E_{F_Z^\nu \sim \Q_Z^\nu} \left[ \log  \Q_Z^\nu(F_Z^\nu) \right]  \nonumber \\
       & = - \E_{F_Z^\nu \sim \Q_Z^\nu} \left[ \bar{\ell}(F_Z^\nu) \right]  + H(\Q_Z^\nu),  \label{eq:2422}
    \end{align}
  where $(A)$ follows from the law of total expectation,  $(B)$ from the definition of the KL divergence, and we defined  
\begin{align*}
\bar{\ell}(u) & :=  - \E_{F^\nu \sim  Q^\nu} \left[ \log p(y|F^\nu_X) \mid F_Z^\nu = u \right] - \log \P_Z( u ),  \\
& = -  \E_{F^{\nu} \sim \Q^{\nu} } \big[ \log \mathcal{N}(y| F^{\nu}_X, \sigma^2 I_n ) \mid F_Z^\nu = u ] - \log \P_Z( u ), \quad u \in \mathbb{R}^m, \\
 H(\Q_Z^\nu) & :=  -  \E_{F_Z^\nu \sim \Q_Z^\nu} \left[ \log  \Q_Z^\nu(F_Z^\nu) \right] =   \frac{1}{2} \log {\rm det}( 2 \pi e \Sigma ).
\end{align*}
The $ H(\Q_Z^\nu)$ is the entropy of the distribution $\Q_Z^\nu = \mathcal{N}(\mu, \Sigma)$. 

The optimality conditions for the variational parameters $\mu \in \mathbb{R}^m$ and $\Sigma \in \mathbb{R}^{m \times m}$ are thus given by setting the gradients of \eqref{eq:2422} to zero \citep[Eq.(5)]{khan2021bayesian}; this leads to 
\begin{align}
& \nabla_\mu \E_{F_Z^\nu \sim \Q_Z^\nu} \left[ \bar{\ell}(F_Z^\nu) \right] = \nabla_\mu   H(\Q_Z^\nu) = 0, \label{eq:optimality-mu-2434} \\
& \nabla_\Sigma \E_{F_Z^\nu \sim \Q_Z^\nu} \left[ \bar{\ell}(F_Z^\nu) \right] = \nabla_\Sigma  H(\Q_Z^\nu) = \frac{1}{2} \Sigma^{-1}. \label{eq:optimality-sigma-2435} 
\end{align}
To use these conditions to derive optimal $\mu$ and $\Sigma$, we first analyze the function $\bar{\ell}(u)$.


Note that we have $F_X^\nu | F_Z = u  \sim \mathcal{N}(\mu_u, \Sigma_u)$, where
\begin{align}
    &\mu_u:= k_{XZ} k_{ZZ}^{-1} u \in \mathbb{R}^m, \quad \Sigma_{u}:= k_{XX}- k_{XZ} k_{ZZ}^{-1} k_{ZX} = k_{XX}- q_{XX} \in \mathbb{R}^{m \times m}. \nonumber
\end{align}
From this, one can show that
\begin{equation}
    \E_{F^{\nu} \sim \Q^{\nu} } \big[ \log \mathcal{N}(y| F^{\nu}_X, \sigma^2 I_n ) \mid F_Z^\nu = u ] = \log \mathcal{N}(y \,  | \, \mu_u, \sigma^2 I_n) - \frac{1}{2 \sigma^2} {\rm tr}(k_{XX}- q_{XX}), \nonumber
\end{equation}
and therefore we obtain
\begin{align*}
    \bar{\ell}(u) &= -\log \mathcal{N}(y \,  | \, \mu_u, \sigma^2 I_n) + \frac{1}{2 \sigma^2} {\rm tr}(k_{XX}- q_{XX}) - \log \P_Z( f_Z ) \nonumber \\
    &= \frac{1}{2 \sigma^2} \sum_{i=1}^n \big(y_i - k_Z(x_i)^T k_{ZZ}^{-1} u\big)^2 + \frac{1}{2} u^\top k_{ZZ}^{-1} u + C \\
    &= \frac{1}{2 \sigma^2} \| y- \mu_u \|_2^2 + \frac{1}{2} u^\top k_{ZZ}^{-1} u + C, 
\end{align*}
where we absorb everything that does not depend on $u$ in $C$. 
Thus, the gradient and the Hessian matrix of $\bar{\ell}(u)$ with respect to $u$ are given by
\begin{align}
    \nabla_u \bar{\ell}(u) &= - \frac{1}{\sigma^2} k_{ZZ}^{-1} k_{ZX}^{} (y- \mu_u) + k_{ZZ}^{-1} u  \nonumber \\
    &= - \frac{1}{\sigma^2} k_{ZZ}^{-1} k_{ZX}^{}y + \frac{1}{\sigma^2} k_{ZZ}^{-1} k_{ZX}^{} \mu_u + k_{ZZ}^{-1} u, \nonumber \\
    \nabla_u^2 \bar{\ell}(u) &= \frac{1}{\sigma^2}  k_{ZZ}^{-1} k_{ZX}^{} k_{XZ}^{} k_{ZZ}^{-1} + k_{ZZ}^{-1}. \nonumber
\end{align}


We are ready to derive optimal $\mu$ and $\Sigma$. 
We first use \eqref{eq:optimality-mu-2434} to derive optimal $\mu$, following the derivation in \citet[Eq.(25)]{khan2021bayesian}.
We have 
\begin{align*}
& \nabla_\mu \E_{F_Z^\nu \sim \Q_Z^\nu} \left[ \bar{\ell}(F_Z^\nu) \right]  \stackrel{(A)}{=}  \E_{F_Z^\nu \sim \Q_Z^\nu} \left[ \left. \nabla_u \bar{\ell}(u) \right|_{u = F_Z^\nu} \right] \\
& = \E_{F_Z^\nu \sim \Q_Z^\nu} \left[   - \frac{1}{\sigma^2} k_{ZZ}^{-1} k_{ZX}^{}y + \frac{1}{\sigma^2} k_{ZZ}^{-1} k_{ZX}^{} k_{XZ} k_{ZZ}^{-1} F_Z^\nu + k_{ZZ}^{-1} F_Z^\nu \right] \\
& =  - \frac{1}{\sigma^2} k_{ZZ}^{-1} k_{ZX}^{}y + \frac{1}{\sigma^2} k_{ZZ}^{-1} k_{ZX}^{} k_{XZ} k_{ZZ}^{-1} \mu + k_{ZZ}^{-1} \mu. 
\end{align*}
where $(A)$ follows Bonnet's theorem. 
Therefore \eqref{eq:optimality-mu-2434} leads to
\begin{equation}
    \mu = k_{ZZ} (\sigma^2 k_{ZZ} + k_{ZX}k_{XZ})^{-1} k_{ZX} y,  \nonumber
\end{equation}
which recovers the optimal $\mu^*$ in  \eqref{eq:optimal-mu}.


 We next use \eqref{eq:optimality-sigma-2435}  to derive optimal $\Sigma$, following the derivation in \citet[Eq.(26)]{khan2021bayesian}. We have 
 \begin{align*}
 \nabla_\Sigma \E_{F_Z^\nu \sim \Q_Z^\nu} \left[ \bar{\ell}(F_Z^\nu) \right] \stackrel{(A)}{=} \frac{1}{2}  \E_{F_Z^\nu \sim \Q_Z^\nu} \left[ \left.   \nabla_u^2 \bar{\ell}(u) \right|_{u = F_Z^\nu } \right] = \frac{1}{\sigma^2}  k_{ZZ}^{-1} k_{ZX}^{} k_{XZ}^{} k_{ZZ}^{-1} + k_{ZZ}^{-1}.
 \end{align*}
where $(A)$ holds from \citet[Eq.(A.3)]{opper2009variational}. 
Thus  \eqref{eq:optimality-sigma-2435} leads to 
\begin{align}
    \Sigma &= \Big( \frac{1}{\sigma^2}  k_{ZZ}^{-1} k_{ZX}^{} k_{XZ}^{} k_{ZZ}^{-1} + k_{ZZ}^{-1} \Big)^{-1}  = k_{ZZ}^{} \big( \frac{1}{\sigma^2} k_{ZX}^{} k_{XZ}^{} + k_{ZZ}^{}  \big)^{-1} k_{ZZ}, \nonumber
\end{align}
recovering the optimal $\Sigma^*$ in \eqref{eq:optimal-sigma}.


}

\section{Proofs}

\subsection{Proof of Lemma \ref{lemma_M=H_q}}\label{ap_proof_Hq=M}

\begin{proof}
We first show $\mathcal{H}_q = M$ as a set of functions.
First note that 
$$
q(\cdot, x) = k_Z(\cdot)^\top k_{ZZ}^{-1} k_Z(x)  = P_M \big(k(\cdot, x)\big), \quad \forall x \in \mathcal{X}.
$$
Define $\mathcal{H}_{0,q}$ as the vector space
\begin{align*}
	&\mathcal{H}_{0,q}:= \Big\{ f = \sum_{i=1}^{n} \alpha_i q(\cdot, d_i)   \mid  n \in \mathbb{N},~  \alpha = (\alpha_1,...,\alpha_n)^\top \in \mathbb{R}^n,~ 
	D=(d_1,...,d_n) \in \mathcal{X}^n 
	\Big\}. 
\end{align*}
Let $f := \sum_{i=1}^{n} \alpha_i q(\cdot, d_i) \in \mathcal{H}_{0,q}$ be arbitrary.
Since $q(\cdot,d_i) = P_M \big(k(\cdot,d_i)\big) \in M$, we have $f = \sum_{i=1}^{n} \alpha_i q(\cdot, d_i) \in M$ by the linearity of $M$.
Therefore $\mathcal{H}_{0,q} \subset M$.  
On the other hand, for any $f = \sum_{j=1}^m \beta_j k(\cdot, z_j) \in M$ with some $\beta_1, \dots, \beta_m \in \mathbb{R}$, we have $f = \sum_{j=1}^m \beta_j k(\cdot, z_j) =  \sum_{j=1}^m \beta_j P_M(k(\cdot, z_j)) = \sum_{j=1}^m \beta_j q(\cdot, z_j) \in \mathcal{H}_{q, 0}$. 
Therefore $M \subset \mathcal{H}_{q, 0}$.
Thus we have shown $\mathcal{H}_{0,q} = M$ as a set.

Note that the RKHS $\mathcal{H}_q$ is by definition the closure\footnote{\textcolor{black}{The closure of a subset $A$ in a Hilbert space $\mathcal{H}$ is defined as $\bar{A}:=\{h \in \mathcal{H} \, | \, \text{ there exists a sequence }  \{a_n\}_{n=1}^\infty \subset A$ \text{ such that } $\lim_{n \to \infty } \| a_n - h \|_{\mathcal{H}} = 0 \}$.}} of $\mathcal{H}_{0,q}$ with respect to the norm 
\begin{align*}
& \left\| \sum_{i=1}^{n} \alpha_i q(\cdot, d_i) \right\|_{\mathcal{H}_q}^2 =  \alpha^\top q_{DD} \alpha =  \alpha^\top k_{DZ}k_{ZZ}^{-1}k_{ZD} \alpha  =  \alpha^\top k_{DZ}k_{ZZ}^{-1} k_{ZZ} k_{ZZ}^{-1} k_{ZD} \alpha \\
&= \left\|   k_Z(\cdot)^\top  k_{ZZ}^{-1} k_{ZD} \alpha \right\|_{\mathcal{H}_k}^2   = \left\| q_D(\cdot)^\top \alpha  \right\|_{\mathcal{H}_k}^2 = \left\| \sum_{i=1}^{n} \alpha_i q(\cdot, d_i)   \right\|_{\mathcal{H}_k}^2,
\end{align*}
which coincides with the norm of $\mathcal{H}_k$. 
As $\mathcal{H}_q = M$ is a finite-dimensional subspace of $\mathcal{H}_k$, it is closed and therefore 
\begin{equation*}
	\mathcal{H}_q =\overline{\mathcal{H}_{0,q}} = \overline{M} = M.
\end{equation*}
where the closure is with respect to the norm $\| \cdot \|_{\mathcal{H}_k} = \| \cdot \|_{\mathcal{H}_q} $.

Next we show that the scalar products on $M$ and $\mathcal{H}_q$ also coincide. Take arbitrary $f$ and $g$ from $\mathcal{H}_q$. As $\mathcal{H}_q = \mathcal{H}_{0,q}$, we find a representation of the form 
\begin{align*}
	f &= q_{D}(\cdot)^\top \alpha = k_{Z}(\cdot) ^\top k_{ZZ}^{-1} k_{ZD}^{} \alpha = k_{Z}(\cdot) ^\top \tilde{\alpha} \\ 
	g &= q_{E}(\cdot)^\top \beta = k_{Z}(\cdot) ^\top k_{ZZ}^{-1} k_{ZE}^{} \beta = k_{Z}(\cdot) ^\top \tilde{\beta},
\end{align*}
where $D=(d_1,...,d_n) \in \X^n$,  $E= (e_1,...,e_\ell) \in \X^\ell$, $\alpha \in \R^n$, $\beta \in \R^\ell$, $\tilde{\alpha}:= k_{ZZ}^{-1} k_{ZD}^{} \alpha$ and $\tilde{\beta}:= k_{ZZ}^{-1} k_{ZE}^{} \beta$. This leads to 
\begin{equation*}
	\langle f,g \rangle_{\mathcal{H}_q} = \alpha^\top q_{DE} \beta = \alpha^\top k_{DZ}^{} k_{ZZ}^{-1} k_{ZE}^{} \beta
\end{equation*}
and 
\begin{equation*}
	\langle f,g \rangle_{\mathcal{H}_k} = \tilde{\alpha}^\top k_{ZZ} \tilde{\beta} = \alpha^\top k_{DZ}^{} k_{ZZ}^{-1} k_{ZE}^{} \beta,
\end{equation*}
which shows that the scalar products are the same. 

\end{proof}

\subsection{Proof of Theorem \ref{theo:ELBO-expression}}
\label{sec:proof-elbo-interp}

\begin{proof} Recall that the ELBO, which we defined in \eqref{eq_def_L} is given as
   \begin{align} 
        \mathcal{L}(\nu) = - KL \big( \Q_Z^\nu \  \| \  \P_Z  \big) + \E_{F^\nu \sim \Q^\nu} \big[ \log p(y|F^\nu_X) \big].
    \end{align}
The KL-term in $\mathcal{L}$ is tractable as KL-divergence between the two Gaussians $\mathcal{N}(\mu, \Sigma)$ and $\mathcal{N}\big(0,k(Z,Z)\big)$ and given as
\begin{align}
KL \big( \Q_Z^\nu \  \| \  \P_Z  \big) = \frac{1}{2} \left( {\rm tr} (k_{ZZ}^{-1} \Sigma) + \mu^\top k_{ZZ}^{-1} \mu - m + \log \left(  \frac{ {\rm det} k_{ZZ}  }{ {\rm det} \Sigma  } \right) \right). \label{eq:KL-Gaussians_ap}
\end{align}
We now focus on the expected log-likelihood-term. Let $f_Z \in \mathbb{R}^m$ be an arbitrary vector in the support of $\Q^\nu_Z$.
Define a notation for the conditional expectation 
\begin{align*}
F_m(x) := \mathbb{E}[ F(x) \mid F_Z = f_Z ] = \mathbb{E}[ F^\nu(x) \mid F^\nu_Z = f_Z ] .
\end{align*}
where the identity follows from the definition of $F^\nu \sim \Q^\nu$. 
Then,  from the standard bias-variance decomposition argument, we have
\begin{align}
& \sum_{i=1}^n \mathbb{E} \left[ (y_i - F^\nu(x_i))^2 \mid F^\nu_Z = f_Z \right] =  \sum_{i=1}^n \mathbb{E} \left[ (y_i - F (x_i))^2 \mid F_Z = f_Z \right] \nonumber \\
& = \sum_{i=1}^n \mathbb{E} \left[ (y_i -   F_m(x_i))^2  \mid F_Z = f_Z  \right]   + \sum_{i=1}^n \mathbb{E} \left[ ( F_m (x_i) -   F(x_i))^2 \mid F_Z = f_Z \right]  \nonumber \\
& = \sum_{i=1}^n  (y_i -   F_m(x_i))^2   + \sum_{i=1}^n \mathbb{E} \left[ ( F_m (x_i) -   F(x_i))^2 \mid F_Z = f_Z \right]. \label{eq:proof-elbo-two-terms}
\end{align}

Note that, because $F_m(x)$ is the conditional expectation of $F(x)$ given $F_Z = f_Z$, it is equivalent to the kernel interpolator with training data $(z_i, F(z_i))_{i=1}^n$ and can be written as
$$
F_m(x) = k_Z(x)^\top k_{ZZ}^{-1} f_Z.
$$
Therefore 
\begin{align*}
&  \int     (y_i -   F_m(x_i))^2  dQ_Z^\nu (f_Z) 
=   \int  (y_i -    k_Z(x_i)^\top k_{ZZ}^{-1} f_Z )^2  dQ_Z^\nu (f_Z) \\
&=   \int  (y_i -    k_Z(x_i)^\top k_{ZZ}^{-1} \mu )^2  dQ_Z^\nu (f_Z)   + \int  (k_Z(x_i)^\top k_{ZZ}^{-1} \mu -    k_Z(x_i)^\top k_{ZZ}^{-1} f_Z )^2  dQ_Z^\nu (f_Z) \\
& =    (y_i -    k_Z(x_i)^\top k_{ZZ}^{-1} \mu )^2   +  k_Z(x_i)^\top k_{ZZ}^{-1} \Sigma k_{ZZ}^{-1} k_Z(x_i),
\end{align*}
where the last identity follows from $Q_Z^\nu = \mathcal{N}(\mu, \Sigma)$ by definition. 

On the other hand, the second term in \eqref{eq:proof-elbo-two-terms} is the conditional variance of $F(x_i)$ given $F_Z = f_Z$, and thus given by 
\begin{align*}
& \mathbb{E} \left[ ( F_m (x_i) -   F(x_i))^2 \mid F_Z = f_Z \right]  = k(x_i,x_i) - k_Z(x_i)^\top k_{ZZ}^{-1} k_Z(x_i),
\end{align*}
which is independent to the ``observations'' $f_Z$. 
Therefore 
\begin{align*}
& \int \mathbb{E} \left[ ( F_m (x_i) -   F(x_i))^2 \mid F_Z = f_Z \right] dQ^\nu_Z(f_Z)  =  k(x_i,x_i) - k_Z(x_i)^\top k_{ZZ}^{-1} k_Z(x_i).
\end{align*}
Using these identities, we have
\begin{align*}
 & 
\int \left( \sum_{i=1}^n \mathbb{E} \left[ (y_i - F^\nu(x_i))^2 \mid F^\nu_Z = f_Z \right] \right) dQ^\nu_Z(f_Z) \\
& =  \sum_{i=1}^n    (y_i -    k_Z(x_i)^\top k_{ZZ}^{-1} \mu )^2 + \sum_{i=1}^n   k_Z(x_i)^\top k_{ZZ}^{-1} \Sigma k_{ZZ}^{-1} k_Z(x_i) \\
& +     \sum_{i=1}^n   \left( k(x_i,x_i) - k_Z(x_i)^\top k_{ZZ}^{-1} k_Z(x_i) \right).
\end{align*}
The assertion follows from these derived expressions.

\end{proof}

\subsection{Proof of  Lemma \ref{lemma:emp-risk-expression}} \label{sec:proof-lemma-emp-risk-expression}

\begin{proof}
Recall that $\hat{f} :=   k_X(\cdot)^\top (k_{XX} + \sigma^2 I_n)^{-1} y$ is the solution of KRR.
We have 
\begin{align*}
 \hat{f}_X &=  k_{XX} (k_{XX} + \sigma^2 I_n)^{-1} y = (I_n-\sigma^2 (k_{XX} + \sigma^2 I_n)^{-1} ) y,
\end{align*}
where we used the formula $A (A + \sigma^2 I_n)^{-1} = I_n - \sigma^2 (A + \sigma^2 I_n)^{-1}$ that holds for any positive semidefinte matrix $A$.
Now we have
\begin{align*}
 &   \min_{f \in \mathcal{H}_k} \sum_{i=1}^n (y_i - f(x_i))^2 + \sigma^2 \| f \|_{\mathcal{H}_k}^2   =   \| y - \hat{f}_X \|^2  + \sigma^2 \| \hat{f} \|_{\mathcal{H}_k}^2 
\end{align*}  
The first term can be expanded as
\begin{align*}
  &   \| y - \hat{f}_X \|^2  = \| y \|^2 - 2  y^\top \hat{f}_X  + \| \hat{f}_X \|^2    \\
 & = \| y \|^2 - 2 y^\top  (I_n-\sigma^2 (k_{XX} + \sigma^2 I_n)^{-1} ) y  + y^\top (I_n-\sigma^2 (k_{XX} + \sigma^2 I_n)^{-1} )^2 y \\
  & = \| y \|^2 - 2 y^\top  (I_n-\sigma^2 (k_{XX} + \sigma^2 I_n)^{-1} ) y \\
 & + y^\top (I_n-2 \sigma^2 (k_{XX} + \sigma^2 I_n)^{-1} + \sigma^4 (k_{XX} + \sigma^2 I_n)^{-2} ) y \\
 & = \sigma^4 y^\top (k_{XX} + \sigma^2 I_n)^{-2}  y .
\end{align*}
The second term is
\begin{align*}
     &  \sigma^2 \| \hat{f} \|_{\mathcal{H}_k}^2 = \sigma^2  y^\top (k_{XX} + \sigma^2 I_n)^{-1}   k_{XX}(k_{XX} + \sigma^2 I_n)^{-1} y \\
     & = \sigma^2 y^\top (k_{XX} + \sigma^2 I_n)^{-1} ( I_n - \sigma^2 (k_{XX} + \sigma^2 I_n)^{-1} ) y \\
     & = \sigma^2 y^\top (k_{XX} + \sigma^2 I_n)^{-1} y  - \sigma^4 y^\top(k_{XX} + \sigma^2 I_n)^{-2} y.
\end{align*} 
Therefore,
\begin{align*}
    & \| y - \hat{f}_X \|^2  + \sigma^2 \| \hat{f} \|_{\mathcal{H}_k}^2  = \sigma^2 y^\top (k_{XX} + \sigma^2 I_n)^{-1} y.
\end{align*}

\end{proof}

\subsection{Proof of Theorem \ref{theo:nystrom-RKHR-error}}
\label{sec:Proof-nystrom-RKHS-error}

\begin{proof}
We first make preliminaries for proving the theorem.
For a symmetric matrix $B \in \mathbb{R}^{n \times n}$ with $n \in \mathbb{N}$, denote by $\lambda_1(B) \geq \cdots \geq \lambda_n(B)$ its eigenvalues with multiplicities in the decreasing order. 
For any symmetric and positive semi-definite (SPSD) matrix $A\in \mathbb{R}^{n\times n}$  and any $B \in \mathbb{R}^{n\times n}$, we have \cite[see][]{saniuk1987matrix} 
\begin{equation}\label{eq_trace_ineq}
    {\rm tr}(AB) \le {\rm tr}(A) \|B\|_{\rm op},
\end{equation}
where \textcolor{black}{$\| B \|_{\rm op} := \sup_{v \in \mathbb{R}:~\| v \| \leq 1} \| B v \|$ denotes the operator norm (or spectral norm). If $B$ is symmetric, we have $\| B \|_{\rm op} = \max(|\lambda_{1}(B)|, |\lambda_n(B)|)$.}

For any SPSD matrix $A \in \mathbb{R}^{n \times n}$ and any symmetric and negative semi-definite (SNSD) matrix $B \in \mathbb{R}^{n \times n}$, we have 
\begin{align}
{\rm tr}(AB) &= {\rm tr}(BA) 
= {\rm tr} \big( A^{1/2} B A^{1/2} \big)  = \sum_{i=1}^{n} \lambda_i(A^{1/2} B A^{1/2})
\le 0, \label{eq_trace_neg}
\end{align}
where the inequality follows from the fact that $A^{1/2} B A^{1/2}$ is SNSD and hence all its eigenvalues are non-positive. 

We also use the following short hand notation
\begin{align*}
    &K:= k_{XX}, \quad Q:=q_{XX}, \quad \tilde{K}:= K+n\lambda I_n, \quad \tilde{Q}:= Q+n\lambda I_n, \\
        &\alpha = \tilde{K}^{-1}y,\quad\beta = \tilde{Q}^{-1} y, \quad \tilde{\beta} = \big( n \lambda k_{ZZ} + k_{ZX} k_{XZ}\big)^{-1} k_{ZX}  y.
\end{align*}
Note that the matrices $K, Q, \tilde{K}, \tilde{Q}$ are SPSD. It holds that
\begin{align*}
   & \| \tilde{K}^{-1} \|_{\rm op} \le \frac{1}{n \lambda},\quad  \| \tilde{Q}^{-1} \|_{\rm op} \le \frac{1}{n \lambda},\quad  \|K \tilde{K}^{-1} \|_{\rm op} \le 1,\quad \|Q \tilde{Q}^{-1} \|_{\rm op} \le 1.
\end{align*}
Using the above notation, the KRR estimator $\hat{f}$ and the Nystr\"om approximation $\bar{f}$ can be written for any $x \in \mathcal{X}$ as
\begin{align*}
    \hat{f}(x) &= k_X(x)^\top \tilde{K}^{-1} y, \\ 
    \bar{f}(x) &= q_X(x)^\top  \tilde{Q}^{-1}y  = q_X(x)^\top  \beta \\
    & = k_Z(x)^\top \big( n \lambda k_{ZZ} + k_{ZX} k_{XZ}\big)^{-1} k_{ZX}  y = k_Z(x)^\top \tilde{\beta}.
\end{align*}
We will use the following identity:
\begin{align*}
    k_{XZ} \tilde{\beta} = \bar{f}_X = q_{XX} \tilde{Q}^{-1}y = Q \tilde{Q}^{-1} y.
\end{align*}

With these preparations, we now prove the assertion.
First we have
\begin{align*}
    \| \hat{f} - \bar{f} \|^2_{\mathcal{H}_k} &= \| \hat{f} \|^2_{\mathcal{H}_k} - 2 \langle \hat{f}, \bar{f} \rangle_{\mathcal{H}_k} + \| \bar{f} \|^2_{\mathcal{H}_k} \\
    &= \| \hat{f} \|^2_{\mathcal{H}_k} - 2 \langle \hat{f}, \bar{f} \rangle_{\mathcal{H}_k} + \| \bar{f} \|^2_{\mathcal{H}_q},
\end{align*}
where we used $\| \bar{f} \|^2_{\mathcal{H}_q}$ = $\| \bar{f} \|^2_{\mathcal{H}_k}$, which holds from $\bar{f} \in M $ and Lemma \ref{lemma_M=H_q}.
The expression is equal to
\begin{align*} 
 &= \alpha^\top K \alpha - 2 \alpha^\top k_{XZ} \tilde{\beta} + \beta^\top Q \beta \\ 
&= {\rm tr}(K \alpha \alpha^\top) - 2 {\rm tr}(k_{XZ} \tilde{\beta} \alpha^\top) + {\rm tr}(Q \beta \beta^\top ) \\
&= {\rm tr}(K \tilde{K}^{-1}y y^\top \tilde{K}^{-1} ) - {\rm tr}(Q \tilde{Q}^{-1} y y^\top \tilde{K}^{-1} )   + {\rm tr}(Q \tilde{Q}^{-1} y y^\top \tilde{Q}^{-1} ) -  {\rm tr}(Q \tilde{Q}^{-1} y y^\top \tilde{K}^{-1} ) \\
&= {\rm tr}\big( (K \tilde{K}^{-1} - Q \tilde{Q}^{-1})yy^\top \tilde{K}^{-1} \big)  + {\rm tr} \big( Q \tilde{Q}^{-1} y y^\top (\tilde{Q}^{-1}-\tilde{K}^{-1}) \big) \\
&\le {\rm tr}(K \tilde{K}^{-1} - Q  \tilde{Q}^{-1} ) \|y y^\top \|_{\rm op} \|\tilde{K}^{-1} \|_{\rm op}   + {\rm tr}(\tilde{Q}^{-1}-\tilde{K}^{-1}) \| Q \tilde{Q}^{-1}\|_{\rm op} \| y y^\top \|_{\rm op} \\
&\le  \frac{1}{n \lambda} {\rm tr}(K \tilde{K}^{-1} - Q \tilde{Q}^{-1} ) \| y \|^2  + {\rm tr}(\tilde{Q}^{-1}-\tilde{K}^{-1}) \| y \|^2 \\
&= \frac{1}{n \lambda} \Big( {\rm tr}\big( (K-Q)\tilde{K}^{-1}  \big) + {\rm tr}\big( Q (\tilde{K}^{-1}- \tilde{Q}^{-1} )\big) \Big) \| y \|^2  + {\rm tr}(\tilde{Q}^{-1}(K-Q) \tilde{K}^{-1} ) \|y \|^2
\end{align*}
Since $Q$ is SPSD and $\tilde{K}^{-1}-\tilde{Q}^{-1}$ is SNSD, we have 
\begin{equation*}
    {\rm tr}\big( Q (\tilde{K}^{-1}- \tilde{Q}^{-1} )\big) \le 0
\end{equation*}
due to \eqref{eq_trace_neg}. Using this and \eqref{eq_trace_ineq}, we obtain
\begin{align*}
&\le \frac{1}{n \lambda} {\rm tr}(K-Q) \|\tilde{K}^{-1} \|_{\rm op} \|y \|^2  + {\rm tr}(K-Q) \|\tilde{K}^{-1}\|_{\rm op} \|\tilde{Q}^{-1} \|_{\rm op} \|y \|^2 \\
&\le \frac{2}{(n\lambda)^2} {\rm tr}(K-Q) \|y \|^2,
\end{align*}
which concludes the proof. 
\end{proof}

\vskip 0.2in
\bibliography{Thesis_bibliography}

\begin{thebibliography}{64}
\providecommand{\natexlab}[1]{#1}
\providecommand{\url}[1]{\texttt{#1}}
\expandafter\ifx\csname urlstyle\endcsname\relax
  \providecommand{\doi}[1]{doi: #1}\else
  \providecommand{\doi}{doi: \begingroup \urlstyle{rm}\Url}\fi

\bibitem[Adam et~al.(2020)Adam, Eleftheriadis, Artemev, Durrande, and
  Hensman]{adam2020doubly}
Vincent Adam, Stefanos Eleftheriadis, Artem Artemev, Nicolas Durrande, and
  James Hensman.
\newblock Doubly sparse variational gaussian processes.
\newblock In \emph{International Conference on Artificial Intelligence and
  Statistics}, pages 2874--2884. PMLR, 2020.

\bibitem[Affandi et~al.(2013)Affandi, Kulesza, Fox, and
  Taskar]{affandi2013nystrom}
Raja~Hafiz Affandi, Alex Kulesza, Emily Fox, and Ben Taskar.
\newblock Nystr\"om approximation for large-scale determinantal processes.
\newblock In \emph{Artificial Intelligence and Statistics}, pages 85--98. PMLR,
  2013.

\bibitem[Aronszajn(1950)]{aronszajn1950theory}
Nachman Aronszajn.
\newblock Theory of reproducing kernels.
\newblock \emph{Transactions of the American Mathematical Society}, 68\penalty0
  (3):\penalty0 337--404, 1950.

\bibitem[Bach(2013)]{bach2013sharp}
Francis Bach.
\newblock Sharp analysis of low-rank kernel matrix approximations.
\newblock In \emph{Conference on Learning Theory}, pages 185--209. PMLR, 2013.

\bibitem[Bauer(2020)]{bauer2020advances}
Matthias Bauer.
\newblock \emph{{Advances in Probabilistic Modelling: Sparse Gaussian
  Processes, Autoencoders, and Few-shot Learning}}.
\newblock PhD thesis, University of Cambridge, 2020.

\bibitem[Bauer et~al.(2016)Bauer, van~der Wilk, and
  Rasmussen]{bauer2016understanding}
Matthias Bauer, Mark van~der Wilk, and Carl~Edward Rasmussen.
\newblock {Understanding probabilistic sparse Gaussian process approximations}.
\newblock \emph{Advances in neural information processing systems}, 29, 2016.

\bibitem[Belabbas and Wolfe(2009)]{belabbas2009spectral}
Mohamed-Ali Belabbas and Patrick~J Wolfe.
\newblock Spectral methods in machine learning and new strategies for very
  large datasets.
\newblock \emph{Proceedings of the National Academy of Sciences}, 106\penalty0
  (2):\penalty0 369--374, 2009.

\bibitem[Berlinet and {Thomas-Agnan}(2004)]{Berlinet2004}
A.~Berlinet and C.~{Thomas-Agnan}.
\newblock \emph{{Reproducing Kernel Hilbert Spaces in Probability and
  Statistics}}.
\newblock Kluwer, 2004.

\bibitem[Burt et~al.(2019)Burt, Rasmussen, and Van Der~Wilk]{burt2019rates}
David Burt, Carl~Edward Rasmussen, and Mark Van Der~Wilk.
\newblock Rates of convergence for sparse variational {G}aussian process
  regression.
\newblock In \emph{Proceedings of the 36th International Conference on Machine
  Learning}, pages 862--871, 2019.

\bibitem[Burt et~al.(2020)Burt, Rasmussen, and van~der Wilk]{BurtJMLR20}
David~R. Burt, Carl~Edward Rasmussen, and Mark van~der Wilk.
\newblock Convergence of sparse variational inference in {G}aussian processes
  regression.
\newblock \emph{Journal of Machine Learning Research}, 21\penalty0
  (131):\penalty0 1--63, 2020.

\bibitem[Chen and Yang(2021)]{chen2021fast}
Yifan Chen and Yun Yang.
\newblock Fast statistical leverage score approximation in kernel ridge
  regression.
\newblock In \emph{International Conference on Artificial Intelligence and
  Statistics}, pages 2935--2943. PMLR, 2021.

\bibitem[Cortes et~al.(2010)Cortes, Mohri, and Talwalkar]{cortes2010impact}
Corinna Cortes, Mehryar Mohri, and Ameet Talwalkar.
\newblock On the impact of kernel approximation on learning accuracy.
\newblock In \emph{Proceedings of the Thirteenth International Conference on
  Artificial Intelligence and Statistics}, pages 113--120. JMLR Workshop and
  Conference Proceedings, 2010.

\bibitem[Csat{\'o} and Opper(2002)]{csato2002sparse}
Lehel Csat{\'o} and Manfred Opper.
\newblock Sparse on-line {G}aussian processes.
\newblock \emph{Neural Computation}, 14\penalty0 (3):\penalty0 641--668, 2002.

\bibitem[Csisz{\'a}r(1975)]{csiszar1975divergence}
Imre Csisz{\'a}r.
\newblock I-divergence geometry of probability distributions and minimization
  problems.
\newblock \emph{The annals of probability}, pages 146--158, 1975.

\bibitem[Derezinski et~al.(2020)Derezinski, Khanna, and
  Mahoney]{Derezinski2020Improved}
Michal Derezinski, Rajiv Khanna, and Michael~W Mahoney.
\newblock Improved guarantees and a multiple-descent curve for {Column Subset
  Selection} and the {N}ystr\"om method.
\newblock In H.~Larochelle, M.~Ranzato, R.~Hadsell, M.~F. Balcan, and H.~Lin,
  editors, \emph{Advances in Neural Information Processing Systems}, volume~33,
  pages 4953--4964. Curran Associates, Inc., 2020.

\bibitem[Donsker and Varadhan(1975)]{donsker1975asymptotic}
Monroe~D Donsker and SR~Srinivasa Varadhan.
\newblock Asymptotic evaluation of certain markov process expectations for
  large time, i.
\newblock \emph{Communications on Pure and Applied Mathematics}, 28\penalty0
  (1):\penalty0 1--47, 1975.

\bibitem[Drineas and Mahoney(2005)]{drineas2005nystrom}
Petros Drineas and Michael~W. Mahoney.
\newblock On the {N}ystr\"om method for approximating a {G}ram matrix for
  improved kernel-based learning.
\newblock \emph{Journal of Machine Learning Research}, 6\penalty0
  (72):\penalty0 2153--2175, 2005.

\bibitem[Dutordoir et~al.(2020)Dutordoir, Durrande, and
  Hensman]{dutordoir20-sparse}
Vincent Dutordoir, Nicolas Durrande, and James Hensman.
\newblock Sparse {G}aussian processes with spherical harmonic features.
\newblock In Hal~Daumé III and Aarti Singh, editors, \emph{Proceedings of the
  37th International Conference on Machine Learning}, volume 119 of
  \emph{Proceedings of Machine Learning Research}, pages 2793--2802. PMLR,
  2020.

\bibitem[El~Alaoui and Mahoney(2015)]{alaoui2015fast}
Ahmed El~Alaoui and Michael~W Mahoney.
\newblock Fast randomized kernel ridge regression with statistical guarantees.
\newblock In \emph{Advances in Neural Information Processing Systems}, pages
  775--783, 2015.

\bibitem[Fowlkes et~al.(2004)Fowlkes, Belongie, Chung, and
  Malik]{fowlkes2004spectral}
Charless Fowlkes, Serge Belongie, Fan Chung, and Jitendra Malik.
\newblock Spectral grouping using the {N}ystr\"om method.
\newblock \emph{IEEE Transactions on Pattern Analysis and Machine
  Intelligence}, 26\penalty0 (2):\penalty0 214--225, 2004.

\bibitem[Gittens and Mahoney(2016)]{gittens2016revisiting}
Alex Gittens and Michael~W Mahoney.
\newblock Revisiting the nystr{\"o}m method for improved large-scale machine
  learning.
\newblock \emph{The Journal of Machine Learning Research}, 17\penalty0
  (1):\penalty0 3977--4041, 2016.

\bibitem[Hensman et~al.(2013)Hensman, Fusi, and Lawrence]{hensman2013gaussian}
James Hensman, Nicol{\`o} Fusi, and Neil~D Lawrence.
\newblock Gaussian processes for big data.
\newblock In \emph{Proceedings of the Twenty-Ninth Conference on Uncertainty in
  Artificial Intelligence}, pages 282--290, 2013.

\bibitem[Hensman et~al.(2015{\natexlab{a}})Hensman, Matthews, and
  Ghahramani]{hensman2015scalable}
James Hensman, Alexander Matthews, and Zoubin Ghahramani.
\newblock Scalable variational {G}aussian process classification.
\newblock In \emph{Artificial Intelligence and Statistics}, pages 351--360.
  PMLR, 2015{\natexlab{a}}.

\bibitem[Hensman et~al.(2015{\natexlab{b}})Hensman, Matthews, Filippone, and
  Ghahramani]{Hensman2015MCMC}
James Hensman, Alexander~G Matthews, Maurizio Filippone, and Zoubin Ghahramani.
\newblock {MCMC} for variationally sparse {G}aussian processes.
\newblock In C.~Cortes, N.~Lawrence, D.~Lee, M.~Sugiyama, and R.~Garnett,
  editors, \emph{Advances in Neural Information Processing Systems}, volume~28.
  Curran Associates, Inc., 2015{\natexlab{b}}.

\bibitem[Hensman et~al.(2018)Hensman, Durrande, and
  Solin]{hensman2017variational}
James Hensman, Nicolas Durrande, and Arno Solin.
\newblock Variational {F}ourier features for {G}aussian processes.
\newblock \emph{Journal of Machine Learning Research}, 18\penalty0
  (151):\penalty0 1--52, 2018.

\bibitem[Hofmann et~al.(2008)Hofmann, Sch{\"o}lkopf, and
  Smola]{hofmann2008kernel}
Thomas Hofmann, Bernhard Sch{\"o}lkopf, and Alexander~J Smola.
\newblock Kernel methods in machine learning.
\newblock \emph{Annals of Statistics}, 36\penalty0 (3):\penalty0 1171--1220,
  2008.

\bibitem[Jordan et~al.(1999)Jordan, Ghahramani, Jaakkola, and
  Saul]{jordan1999introduction}
Michael~I Jordan, Zoubin Ghahramani, Tommi~S Jaakkola, and Lawrence~K Saul.
\newblock An introduction to variational methods for graphical models.
\newblock \emph{Machine learning}, 37\penalty0 (2):\penalty0 183--233, 1999.

\bibitem[Kanagawa et~al.(2018)Kanagawa, Hennig, Sejdinovic, and
  Sriperumbudur]{kanagawa2018gaussian}
Motonobu Kanagawa, Philipp Hennig, Dino Sejdinovic, and Bharath~K
  Sriperumbudur.
\newblock Gaussian processes and kernel methods: A review on connections and
  equivalences.
\newblock \emph{arXiv preprint arXiv:1807.02582}, 2018.

\bibitem[Khan and Rue(2021)]{khan2021bayesian}
Mohammad~Emtiyaz Khan and H{\aa}vard Rue.
\newblock {The Bayesian Learning Rule}.
\newblock \emph{arXiv preprint arXiv:2107.04562}, 2021.

\bibitem[Kimeldorf and Wahba(1970)]{kimeldorf1970correspondence}
G.~S. Kimeldorf and G.~Wahba.
\newblock A correspondence between {B}ayesian estimation on stochastic
  processes and smoothing by splines.
\newblock \emph{The Annals of Mathematical Statistics}, 41\penalty0
  (2):\penalty0 495--502, 1970.

\bibitem[Knoblauch et~al.(2019)Knoblauch, Jewson, and
  Damoulas]{knoblauch2019generalized}
Jeremias Knoblauch, Jack Jewson, and Theodoros Damoulas.
\newblock Generalized variational inference: Three arguments for deriving new
  posteriors.
\newblock \emph{arXiv preprint arXiv:1904.02063}, 2019.

\bibitem[Kumar et~al.(2009)Kumar, Mohri, and Talwalkar]{kumar2009ensemble}
Sanjiv Kumar, Mehryar Mohri, and Ameet Talwalkar.
\newblock Ensemble {N}ystr\"om method.
\newblock In Y.~Bengio, D.~Schuurmans, J.~Lafferty, C.~Williams, and
  A.~Culotta, editors, \emph{Advances in Neural Information Processing
  Systems}, volume~22. Curran Associates, Inc., 2009.

\bibitem[Kumar et~al.(2012)Kumar, Mohri, and Talwalkar]{Kumar12Sampling}
Sanjiv Kumar, Mehryar Mohri, and Ameet Talwalkar.
\newblock Sampling methods for the nystr\"om method.
\newblock \emph{Journal of Machine Learning Research}, 13\penalty0
  (34):\penalty0 981--1006, 2012.

\bibitem[Leibfried et~al.(2020)Leibfried, Dutordoir, John, and
  Durrande]{leibfried2020tutorial}
Felix Leibfried, Vincent Dutordoir, ST~John, and Nicolas Durrande.
\newblock A tutorial on sparse {G}aussian processes and variational inference.
\newblock \emph{arXiv preprint arXiv:2012.13962}, 2020.

\bibitem[Li et~al.(2016)Li, Jegelka, and Sra]{li2016fast}
Chengtao Li, Stefanie Jegelka, and Suvrit Sra.
\newblock Fast {DPP} sampling for {N}ystr\"om with application to kernel
  methods.
\newblock In Maria~Florina Balcan and Kilian~Q. Weinberger, editors,
  \emph{Proceedings of The 33rd International Conference on Machine Learning},
  volume~48 of \emph{Proceedings of Machine Learning Research}, pages
  2061--2070, New York, New York, USA, 20--22 Jun 2016. PMLR.

\bibitem[Matthews et~al.(2016)Matthews, Hensman, Turner, and
  Ghahramani]{matthews2016sparse}
Alexander G de~G Matthews, James Hensman, Richard Turner, and Zoubin
  Ghahramani.
\newblock On sparse variational methods and the {Kullback-Leibler} divergence
  between stochastic processes.
\newblock In \emph{Artificial Intelligence and Statistics}, pages 231--239,
  2016.

\bibitem[Meanti et~al.(2020)Meanti, Carratino, Rosasco, and
  Rudi]{Meanti_NeurIPS2020_kernel}
Giacomo Meanti, Luigi Carratino, Lorenzo Rosasco, and Alessandro Rudi.
\newblock Kernel methods through the roof: Handling billions of points
  efficiently.
\newblock In H.~Larochelle, M.~Ranzato, R.~Hadsell, M.~F. Balcan, and H.~Lin,
  editors, \emph{Advances in Neural Information Processing Systems}, volume~33,
  pages 14410--14422. Curran Associates, Inc., 2020.
\newblock URL
  \url{https://proceedings.neurips.cc/paper/2020/file/a59afb1b7d82ec353921a55c579ee26d-Paper.pdf}.

\bibitem[Musco and Musco(2017)]{musco2016recursive}
Cameron Musco and Christopher Musco.
\newblock Recursive sampling for the {N}ystr\"om method.
\newblock In I.~Guyon, U.~V. Luxburg, S.~Bengio, H.~Wallach, R.~Fergus,
  S.~Vishwanathan, and R.~Garnett, editors, \emph{Advances in Neural
  Information Processing Systems}, volume~30. Curran Associates, Inc., 2017.

\bibitem[Nieman et~al.(2022)Nieman, Szabo, and van Zanten]{Nieman22Contraction}
Dennis Nieman, Botond Szabo, and Harry van Zanten.
\newblock {Contraction rates for sparse variational approximations in Gaussian
  process regression}.
\newblock \emph{Journal of Machine Learning Research}, 23\penalty0
  (205):\penalty0 1--26, 2022.
\newblock URL \url{http://jmlr.org/papers/v23/21-1128.html}.

\bibitem[Opper and Archambeau(2009)]{opper2009variational}
Manfred Opper and C{\'e}dric Archambeau.
\newblock {The variational Gaussian approximation revisited}.
\newblock \emph{Neural Computation}, 21\penalty0 (3):\penalty0 786--792, 2009.

\bibitem[Parzen(1961)]{Par61}
E.~Parzen.
\newblock An approach to time series analysis.
\newblock \emph{The Annals of Mathematical Statistics}, 32\penalty0
  (4):\penalty0 951--989, 1961.

\bibitem[Qui{\~n}onero-Candela and Rasmussen(2005)]{quinonero2005unifying}
Joaquin Qui{\~n}onero-Candela and Carl~Edward Rasmussen.
\newblock A unifying view of sparse approximate {G}aussian process regression.
\newblock \emph{Journal of Machine Learning Research}, 6:\penalty0 1939--1959,
  2005.

\bibitem[Rasmussen and Williams(2006)]{RasmussenWilliams}
C.E. Rasmussen and C.K.I. Williams.
\newblock \emph{{Gaussian Processes for Machine Learning}}.
\newblock MIT Press, 2006.

\bibitem[Rossi et~al.(2021)Rossi, Heinonen, Bonilla, Shen, and
  Filippone]{Rossi21sparse}
Simone Rossi, Markus Heinonen, Edwin Bonilla, Zheyang Shen, and Maurizio
  Filippone.
\newblock Sparse {G}aussian processes revisited: Bayesian approaches to
  inducing-variable approximations.
\newblock In Arindam Banerjee and Kenji Fukumizu, editors, \emph{Proceedings of
  The 24th International Conference on Artificial Intelligence and Statistics},
  volume 130 of \emph{Proceedings of Machine Learning Research}, pages
  1837--1845. PMLR, 13--15 Apr 2021.

\bibitem[Rudi et~al.(2015)Rudi, Camoriano, and Rosasco]{rudi2015less}
Alessandro Rudi, Raffaello Camoriano, and Lorenzo Rosasco.
\newblock Less is more: Nystr{\"o}m computational regularization.
\newblock In \emph{Advances in Neural Information Processing Systems}, pages
  1657--1665, 2015.

\bibitem[Rudi et~al.(2017)Rudi, Carratino, and Rosasco]{rudi2017falkon}
Alessandro Rudi, Luigi Carratino, and Lorenzo Rosasco.
\newblock {FALKON}: An optimal large scale kernel method.
\newblock In I.~Guyon, U.~V. Luxburg, S.~Bengio, H.~Wallach, R.~Fergus,
  S.~Vishwanathan, and R.~Garnett, editors, \emph{Advances in Neural
  Information Processing Systems}, volume~30. Curran Associates, Inc., 2017.

\bibitem[Saniuk and Rhodes(1987)]{saniuk1987matrix}
J~Saniuk and I~Rhodes.
\newblock A matrix inequality associated with bounds on solutions of algebraic
  {R}iccati and {L}yapunov equations.
\newblock \emph{IEEE Transactions on Automatic Control}, 32\penalty0
  (8):\penalty0 739--740, 1987.

\bibitem[Sch{\"o}lkopf and Smola(2002)]{scholkopf2002learning}
Bernhard Sch{\"o}lkopf and Alexander~J Smola.
\newblock \emph{Learning with Kernels: Support Vector Machines, Regularization,
  Optimization, and Beyond}.
\newblock MIT press, 2002.

\bibitem[Sch{\"o}lkopf et~al.(2001)Sch{\"o}lkopf, Herbrich, and
  Smola]{scholkopf2001generalized}
Bernhard Sch{\"o}lkopf, Ralf Herbrich, and Alex~J Smola.
\newblock A generalized representer theorem.
\newblock In \emph{International Conference on Computational Learning Theory},
  pages 416--426. Springer, 2001.

\bibitem[Seeger et~al.(2003)Seeger, Williams, and Lawrence]{seeger2003fast}
Matthias Seeger, Christopher Williams, and Neil Lawrence.
\newblock Fast forward selection to speed up sparse {G}aussian process
  regression.
\newblock In \emph{Artificial Intelligence and Statistics}, 2003.

\bibitem[Shi et~al.(2020)Shi, Titsias, and Mnih]{shi2019sparse}
Jiaxin Shi, Michalis Titsias, and Andriy Mnih.
\newblock Sparse orthogonal variational inference for {G}aussian processes.
\newblock In Silvia Chiappa and Roberto Calandra, editors, \emph{Proceedings of
  the Twenty Third International Conference on Artificial Intelligence and
  Statistics}, volume 108 of \emph{Proceedings of Machine Learning Research},
  pages 1932--1942. PMLR, 26--28 Aug 2020.

\bibitem[Smola and Sch\"olkopf(2000)]{smola2000sparse}
Alex~J Smola and Bernhard Sch\"olkopf.
\newblock Sparse greedy matrix approximation for machine learning.
\newblock In \emph{Proceedings of the Seventeenth International Conference on
  Machine Learning}, pages 911--918, 2000.

\bibitem[Snelson and Ghahramani(2006)]{snelson2006sparse}
Edward Snelson and Zoubin Ghahramani.
\newblock Sparse {G}aussian processes using pseudo-inputs.
\newblock In \emph{Advances in Neural Information Processing Systems}, pages
  1257--1264, 2006.

\bibitem[Snelson and Ghahramani(2007)]{snelson2007local}
Edward Snelson and Zoubin Ghahramani.
\newblock Local and global sparse {G}aussian process approximations.
\newblock In \emph{Artificial Intelligence and Statistics}, pages 524--531,
  2007.

\bibitem[Steinwart and Christmann(2008)]{SteChr2008}
I.~Steinwart and A.~Christmann.
\newblock \emph{Support Vector Machines}.
\newblock Springer, 2008.

\bibitem[Talwalkar et~al.(2008)Talwalkar, Kumar, and
  Rowley]{talwalkar2008large}
Ameet Talwalkar, Sanjiv Kumar, and Henry Rowley.
\newblock Large-scale manifold learning.
\newblock In \emph{2008 IEEE Conference on Computer Vision and Pattern
  Recognition}. IEEE, 2008.

\bibitem[Talwalkar et~al.(2013)Talwalkar, Kumar, Mohri, and
  Rowley]{talwalkar2013large}
Ameet Talwalkar, Sanjiv Kumar, Mehryar Mohri, and Henry Rowley.
\newblock Large-scale {SVD} and manifold learning.
\newblock \emph{Journal of Machine Learning Research}, 14\penalty0
  (60):\penalty0 3129--3152, 2013.

\bibitem[Titsias(2009{\natexlab{a}})]{titsias2009variational}
Michalis Titsias.
\newblock Variational learning of inducing variables in sparse {G}aussian
  processes.
\newblock In \emph{Artificial Intelligence and Statistics}, pages 567--574,
  2009{\natexlab{a}}.

\bibitem[Titsias(2009{\natexlab{b}})]{titsias2009techreport}
Michalis~K Titsias.
\newblock Variational model selection for sparse {G}aussian process regression.
\newblock \emph{Technical Report, University of Manchester, UK},
  2009{\natexlab{b}}.

\bibitem[Tran et~al.(2021)Tran, Milios, Michiardi, and Filippone]{tran21sparse}
Gia-Lac Tran, Dimitrios Milios, Pietro Michiardi, and Maurizio Filippone.
\newblock Sparse within sparse {G}aussian processes using neighbor information.
\newblock In \emph{Proceedings of the Thirty-eighth International Conference on
  Machine Learning}, 2021.

\bibitem[Vakili et~al.(2022)Vakili, Scarlett, Shiu, and
  Bernacchia]{vakili22Improved}
Sattar Vakili, Jonathan Scarlett, Da-Shan Shiu, and Alberto Bernacchia.
\newblock {Improved Convergence Rates for Sparse Approximation Methods in
  Kernel-Based Learning}.
\newblock In Kamalika Chaudhuri, Stefanie Jegelka, Le~Song, Csaba Szepesvari,
  Gang Niu, and Sivan Sabato, editors, \emph{Proceedings of the 39th
  International Conference on Machine Learning}, volume 162 of
  \emph{Proceedings of Machine Learning Research}, pages 21960--21983. PMLR,
  17--23 Jul 2022.
\newblock URL \url{https://proceedings.mlr.press/v162/vakili22a.html}.

\bibitem[Wahba(1990)]{wahba1990spline}
Grace Wahba.
\newblock \emph{Spline Models for Observational Data}.
\newblock SIAM, 1990.

\bibitem[Williams and Seeger(2001)]{williams2001using}
Christopher~KI Williams and Matthias Seeger.
\newblock Using the {N}ystr\"om method to speed up kernel machines.
\newblock In \emph{Advances in Neural Information Processing Systems}, pages
  682--688, 2001.

\bibitem[Wu et~al.(2017)Wu, Poloczek, Wilson, and Frazier]{Jian17BO-gradients}
Jian Wu, Matthias Poloczek, Andrew~G Wilson, and Peter Frazier.
\newblock Bayesian optimization with gradients.
\newblock In I.~Guyon, U.~V. Luxburg, S.~Bengio, H.~Wallach, R.~Fergus,
  S.~Vishwanathan, and R.~Garnett, editors, \emph{Advances in Neural
  Information Processing Systems}, volume~30. Curran Associates, Inc., 2017.

\end{thebibliography}

\end{document}